\documentclass{article}
\usepackage{ijcai24}
\usepackage[table,xcdraw,dvipsnames]{xcolor}
\usepackage{times}
\usepackage{soul}
\usepackage{url}
\usepackage[hidelinks]{hyperref}
\usepackage[utf8]{inputenc}
\usepackage[small]{caption}
\usepackage{graphicx}
\usepackage{amsmath}
\usepackage{amsthm}
\usepackage{booktabs}
\usepackage{algorithm}
\usepackage{algorithmic}
\usepackage[switch]{lineno}
\usepackage{amsfonts}
\usepackage{pifont}
\usepackage{subcaption}
\usepackage{svg}
\usepackage{rotating}

\newtheorem{example}{Example}
\newtheorem{theorem}{Theorem}
\newtheorem{definition}{Definition}
\newtheorem{proposition}{Proposition}

\title{Generative Subspace Adversarial Active Learning for Outlier Detection in Multiple Views of High-dimensional Data}

\author{
Jose Cribeiro-Ramallo
\and
Vadim Arzamasov\and
Federico Matteucci\and
Denis Wambold\and
Klemens Böhm
\affiliations
Karlsruhe Institute of Technology
\emails
\{jose.cribeiro, vadim.arzamasov, federico.matteucci, klemens.boehm\}@kit.edu
\\
denis.wambold@student.kit.edu
}

\begin{document}

\maketitle

\begin{abstract} 
Outlier detection in high-dimensional tabular data is an important task in data mining, essential for many downstream tasks and applications. Existing unsupervised outlier detection algorithms face one or more problems, including inlier assumption (IA), curse of dimensionality (CD), and multiple views (MV). To address these issues, we introduce Generative Subspace Adversarial Active Learning (GSAAL), a novel approach that uses a Generative Adversarial Network with multiple adversaries. These adversaries learn the marginal class probability functions over different data subspaces, while a single generator in the full space models the entire distribution of the inlier class. GSAAL is specifically designed to address the MV limitation while also handling the IA and CD, being the only method to do so. We provide a comprehensive mathematical formulation of MV, convergence guarantees for the discriminators, and scalability results for GSAAL. Our extensive experiments demonstrate the effectiveness and scalability of GSAAL, highlighting its superior performance compared to other popular OD methods, especially in MV scenarios.
\end{abstract}

\section{Introduction}\label{sec::intro}
Outlier detection (OD), a fundamental and widely recognized issue in data mining, involves the identification of anomalous or deviating data points within a dataset. 
Outliers are typically defined as low-probability occurrences within a population~\cite{Wang-2019,Songqiao-2022}.
In the absence of access to the true probability distribution of the data points, OD algorithms rely on the construction of a scoring function. 
Points with higher scores are more likely to be outliers. 
As we will elaborate, existing unsupervised OD algorithms are susceptible to one or more of the following problems, in high-dimensional tabular data scenarios in particular. 
\begin{itemize}
    \item \emph{The inlier assumption} (IA): 
    OD algorithms may make assumptions about what constitutes an inlier, which can be challenging to verify and validate~\cite{LiuYezheng-2020}. 
    \item \emph{The curse of dimensionality} (CD): As the dimensionality of data increases, the challenge of identifying outliers intensifies, often resulting in a diminished effectiveness of certain OD algorithms~\cite{Bellman-1957}.
    \item \emph{Multiple Views} (MV): This alludes to the fact that outliers often are only visible in certain "views" of the data and are hidden in the full space of original features~\cite{Muller-2012-2}. 
\end{itemize}
We now explain these problems one by one.

\emph{The inlier assumption} poses a challenge to algorithms that assume a standard profile of the inlier data.
For example, angle-based algorithms like ABOD \cite{Kriegel-2008} assume that inliers have other inliers at all angles.
Similarly, neighbor-based algorithms like kNN~\cite{Ramaswamy-2000} assume that inliers have other neighboring data 
nearby.
These assumptions influence the scoring process, which is determined by measuring the degree to which a sample diverges from this assumed norm.
While these algorithms can be effective under their specific assumptions, their performance can degrade when these assumptions do not hold~\cite{LiuYezheng-2020}. Ideally, a method should be free of any reliance on inlier assumptions for more robust applicability.

\emph{The curse of dimensionality}~\cite{Bellman-1957} refers to the decrease in the relative proximity of data points as the number of dimensions increases. Simply put, as it increases, the distinctiveness of each point's location decreases, making distances between points less meaningful. This effect is particularly problematic for OD algorithms that rely on distances or on identifying neighbors to detect outliers, such as density- (e.g., LOF~\cite{Breunig-2000}), neighbor- (e.g., kNN~\cite{Ramaswamy-2000}), and cluster-based (e.g., SVDD~\cite[Chapter~2]{aggarwal-2017}) OD algorithms.

\emph{Multiple Views} refers to the phenomenon that certain complex correlations between features are only observable in some feature subspaces~\cite{Muller-2012-2}.
As detailed in~\cite{aggarwal-2017}, this occurs when the dataset contains additional irrelevant features, making some outliers only detectable in certain subspaces. 
In scenarios where multiple subspaces contain different interesting structures, this problem is exacerbated. 
It then becomes increasingly difficult to explain the variability of a data point based solely on its behavior in a single subspace~\cite{Keller-2013}. This problem is different from the curse of dimensionality as it can occur independently of the dimensionality of the dataset, and it can be mitigated with more data points.  We showcase all the presented issues of a detector in the following example:

\begin{figure}
\centering
\includegraphics[width = 1\linewidth]{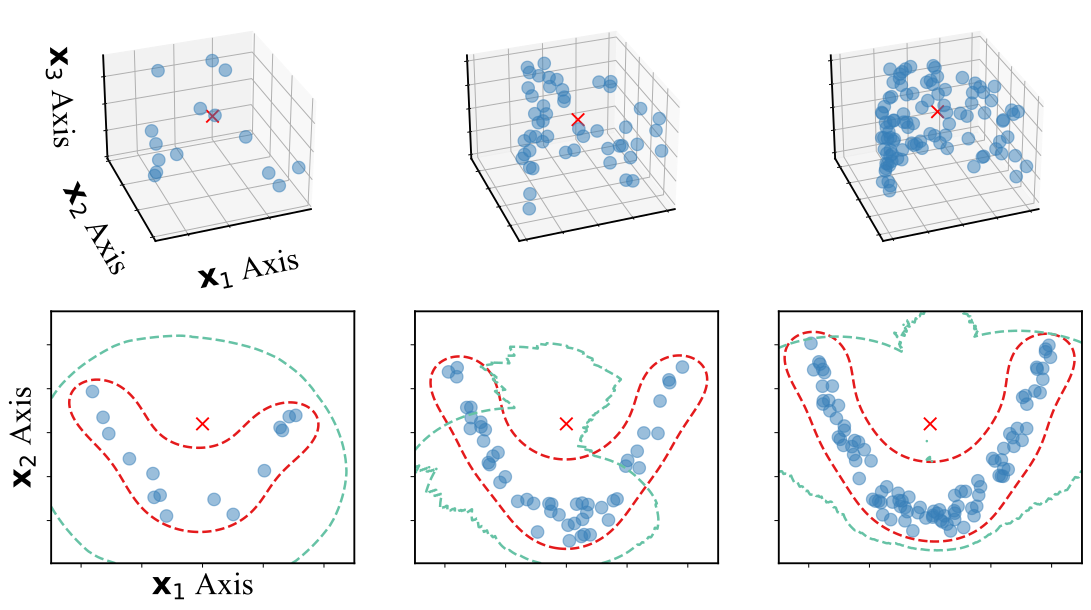}
\caption{Scatterplots of the dataset from example \ref{ex::intro}.}
\label{fig::intro_comparison}
\end{figure}

\begin{example}[Effect of MV, IA and CD]\label{ex::intro}
Consider the random variables $\mathbf{x}_1,\mathbf{x}_2$ and $\mathbf{x}_3$, where $\mathbf{x}_1$ and $\mathbf{x}_2$ are highly correlated and $\mathbf{x}_3$ is just Gaussian noise. Figure \ref{fig::intro_comparison} contains plots with different numbers of realizations for $(\mathbf{x}_1,\mathbf{x}_2,\mathbf{x}_3)$. We also plotted the classification boundaries from both a locality-based method (green) and a cluster-based method (red) in the subspace. If one fits the cluster-based detector in the full space of the data, the outlier shown in the figure (red cross) would not be detectable. However, the outlier is always detectable in the subspace, as we can see. Once we increase the number of samples over $n=1000$, the method detects the outlier in the full space (MV). On the contrary, the locality-based method could not detect the outlier in any tested scenario (MV + IA).
If we increase the dimensionality by adding more features consisting of noise, no method can detect the outlier in the full space (MV + IA + CD). 
\end{example}
We are interested in tackling outlier detection whenever a population exhibits MV, like \cite{Muller-2012-2,Keller-2013,Kriegel-2009} and as showcased in \cite{aggarwal-2017}. 
Particularly, the goal of this paper is to propose the first outlier detection method that explicitly addresses IA, CD, and MV simultaneously. 

As we will explain in the next section, we build on Generative Adversarial Active Learning (GAAL)~\cite{Zhu-2017}, a widely used approach for outlier detection~\cite{LiuYezheng-2020,Guo-2021,Sinha-2019}. 
It involves training a Generative Adversarial Network (GAN) to mimic the distribution of outlier data, and it enhances the discriminator's performance through active learning~\cite{Settles-2009}, leveraging the GAN's data generation capability. 
GAAL methods avoid IA \cite{LiuYezheng-2020} 
and use the multi-layered structure of the GAN to overcome the curse of dimensionality \cite{poggio-2020}.
However, they may miss important subspaces, leading to the multiple views problem. Extending GAAL to also address MV is not immediately obvious. 

\paragraph*{Challenges.} 
Addressing the Multiple Views problem by training a feature ensemble of GAN-based models is not trivial.
(1) The joint training of generators and discriminators in GANs requires careful monitoring to determine the optimal stopping point, a task that becomes daunting for large ensembles. 
(2) The generation of difficult-to-detect points in a subspace 
remains hard~\cite{Steinbuss-2017}. 
While several authors have proposed multi-adversarial architectures for GANs~\cite{Durugkar-2016,choi-2022}, none of them specifically address adversaries tailored to subspaces composed of feature subsets. Furthermore, these methods may not be suitable for GAAL since they do not have convergence guarantees for detectors, as we will explain.

\paragraph*{Contributions.} 
(1) We propose GSAAL (Generative Subspace Adversarial Active Learning), a novel GAAL method that uses multiple adversaries to learn the marginal class probability functions in different data subspaces. 
Each adversary focuses on a single subspace. 
Simultaneously, we train a single generator in the full space to approximate the entire distribution of the inlier class. 
(2) By giving the first mathematical formulation of the ``multiple views'' issue, we showed GSAAL's ability to mitigate the MV problem.
(3) We formulate the novel optimization problem                                 
for GSAAL and give convergence guarantees of each discriminator to the marginal distribution of their respective subspace. Additionally, we give complexity results for the scalability of our method.
(4) Through extensive experimentation, we corroborated our claims regarding scalability and suitability for MV.
(5) We tested our method on the one-class classification task (novelty detection) for outlier detection using 22 popular benchmark datasets. GSAAL outperformed all popular OD algorithms from different families and is orders of magnitude faster in inference than its best competitor.
(6) Our code for the experiments and all examples is publicly available.\footnote{https://github.com/WamboDNS/GSAAL}

Paper outline: Section \ref{sec::rw} reviews related work, Section \ref{sec::gsaal} contains the theoretical results for our method, Section \ref{sec::experiments} features our experimental results, and Section \ref{sec::lim&conc} concludes and addresses limitations. 
\section{Related Work} \label{sec::rw}

\begin{table}
\centering
\begin{tabular}{lllll}
\toprule
Type                & IA & CD         & MV         \\ \midrule
Classical           & \textcolor{BrickRed}{\ding{55}}     & \textcolor{BrickRed}{\ding{55}}     & \textcolor{BrickRed}{\ding{55}}     \\
Subspace            & \textcolor{BrickRed}{\ding{55}}     & \textcolor{ForestGreen}{\checkmark} & \textcolor{ForestGreen}{\checkmark} \\
Generative w/ uniform distribution   & \textcolor{ForestGreen}{\checkmark}    & \textcolor{BrickRed}{\ding{55}}  & \textcolor{BrickRed}{\ding{55}}     \\
Generative w/ param. distribution   & \textcolor{BrickRed}{\ding{55}}    & \textcolor{ForestGreen}{\checkmark}     & \textcolor{BrickRed}{\ding{55}}     \\
Generative w/ subspace behavior & \textcolor{BrickRed}{\ding{55}}     & \textcolor{ForestGreen}{\checkmark} & \textcolor{ForestGreen}{\checkmark} \\
GAAL                & \textcolor{ForestGreen}{\checkmark} & \textcolor{ForestGreen}{\checkmark} & \textcolor{BrickRed}{\ding{55}}     \\ 
\textbf{GSAAL} (Our method)  & \textcolor{ForestGreen}{\checkmark} & \textcolor{ForestGreen}{\checkmark} & \textcolor{ForestGreen}{\checkmark} \\ \bottomrule  
\end{tabular}
\caption{Families of OD methods with the limitations they address.}
\label{tab:summary_rw}
\end{table}    

This section is a brief overview of popular unsupervised outlier detection methods related to our approach. We categorize them based on their ability to address the specific limitations outlined above. Table~\ref{tab:summary_rw} is a comparative summary.

\subsubsection{Classical Methods}
Conventional outlier detection approaches, such as distance-based strategies like LOF and KNN, angle-based techniques like ABOD, and cluster-based methods like SVDD, rely on specific assumptions on the behavior of inlier data. 
They use a scoring function to measure deviations from this assumed norm. 
These methods face the \emph{inlier assumption} limitation by definition. For example, local methods that assume isolated outliers fail 
when several outlying samples fall together. 
In addition, many classical methods, which rely on measuring distances, are susceptible to the \emph{curse of dimensionality}. Both limitations impair the effectiveness and efficiency of these methods~\cite{LiuYezheng-2020}. 

\subsubsection{Subspace Methods}
Subspace-based methods~\cite{Kriegel-2009} operate in lower-dimensional subspaces formed by subsets of features. 
They effectively counteract the curse of dimensionality by focusing on identifying so-called ``subspace outliers''~\cite{Muller-2012}. These outliers, which are prevalent in high-dimensional datasets with many correlated features, 
are often elusive to conventional non-subspace methods~\cite{LiuFey-2008,Muller-2012-2}. However, existing subspace methods inherently operate on specific assumptions on the nature of anomalies in each subspace they explore, and thus face the \emph{inlier assumption} limitation.

\subsubsection{Generative Methods}
A common strategy to mitigate the IA and CD limitations is to reframe the task as a classification task using self-supervision. A prevalent self-supervised technique, particularly for tabular data, is the generation of artificial outliers~\cite{ElYaniv-2006,LiuYezheng-2020}. This method involves distinguishing between actual training data and artificially generated data drawn from a predetermined ``reference distribution''. \cite{Hempstalk-2008} showed that by approximating the class probability of being a real sample, one approximates the probability function of being an inlier.
One then uses this approximation as a scoring function \cite{LiuYezheng-2020}. 
However, it is not easy to find the right reference distribution, and a poor choice can affect OD by much~\cite{Hempstalk-2008}. 

A first approach to this challenge proposed the use of naïve reference distributions by uniformly generating data in the space. 
This approach showed promising results in low-dimensional spaces but failed in high dimensions due to the curse of dimensionality~\cite{Hempstalk-2008}. 
Other approaches, such as assuming parametric distributions for inlier data~\cite[Chapter~2]{aggarwal-2017} or directly generating in susbpaces~\cite{Desir-2013}, can avoid CD when the parametric assumptions are met. 
Methods that generate in the subspaces can model the subspace behavior, additionally tackling the MV limitation.
However, these last two approaches do not address the IA limitation, 
as they make specific assumptions about the behavior of the inlier data.

\subsubsection{Generative Adversarial Active Learning}
According to \cite{Hempstalk-2008}, 
the closer the reference distribution is to the inlier distribution, the better the final approximation to the inlier probability function will be. 
Hence, recent developments in generative methods have focused on learning the
reference distribution in conjunction with the classifier. 
A key approach is the use of Generative Adversarial Networks (GANs), where the generator converges to the inlier distribution~\cite{Goodfellow-2014}.
The most common approaches for this are GAAL-based methods~\cite{LiuYezheng-2020,Guo-2021,Sinha-2019}. These methods differentiate themselves from other GANs for OD by training the detectors using active learning after normal convergence of the GAN~\cite{Schlegl-2017,donahue-2017}. 
The architecture of GAAL inherently addresses the curse of dimensionality, as GANs can incorporate layers designed to manage high-dimensional data \cite{poggio-2020}. 
In practice, GAAL-based methods outperformed all their competitors in their original work. 
However, they overlook the behavior of the data in subspaces and therefore may be susceptible to MV. 

Our method, GSAAL, incorporates several subspace-focused detectors into GAAL. These detectors approximate the marginal inlier probability functions of their subspaces. Thus, GSAAL effectively addresses MV while inheriting GAAL's ability to overcome IA and CD limitations.

\section{Our Method: GSAAL}\label{sec::gsaal}
We first formalize the notion of data exhibiting multiple views. We then use it to design our 
outlier detection method, GSAAL, and give convergence guarantees.
Finally, we derive the runtime complexity of GSAAL. All the proofs and extra derivations can be found in the technical appendix. 

\subsection{Multiple Views}
Several authors~\cite{aggarwal-2017,Muller-2012-2,Keller-2013,Kriegel-2009,LiuFey-2008} have observed that at times the variability of the data can only be explained from its behavior in some subspaces. 
Researchers variably call this problem ``the subspace problem'' \cite{aggarwal-2017,Kriegel-2009} or ``multiple views of the data'' \cite{Muller-2012,Muller-2012-2}.
Previous research has largely focused on practical scenarios, leaving aside the need for a formal definition. 
In response, we propose a unifying definition of ``multiple views'' that provides a foundation for developing methods to address this challenge effectively. 

The problem ``multiple views'' of data (MV) arises from two different effects. 
First, it involves the ability to understand the behavior of a random vector $\mathbf{x}$ by examining lower-dimensional subsets of its components $(\mathbf{x}_1,\dots,\mathbf{x}_d)$. 
Second, it stems from the challenge of insufficient data to obtain an effective scoring function in the full space of $\mathbf{x}$. 
As Example~\ref{ex::intro} shows,
combining these two effects obscures the behavior of the data in the full space. Hence, methods not considering subspaces when building their scoring function may have issues detecting outliers under MV. 
The next definition formalizes the first effect.
\begin{definition}[myopic distribution] \label{def::mv}
    Consider a random vector $\mathbf{x}: \Omega\longrightarrow\mathbb{R}^d$ and 
    $\textit{Diag}_{d\times d}(\{0,1\})$, the set of diagonal binary matrices without the identity. 
    If there exists a random matrix $\mathbf{u}: \Omega\longrightarrow\textit{Diag}_{d\times d}(\{0,1\})$, such that
\begin{equation}\label{eq::mv}
p_\mathbf{x}(x) = p_\mathbf{ux}(ux) \text{ for almost all $x$,}
\end{equation} 
    we say that the distribution of $\mathbf{x}$ is \emph{myopic to the views of} $\mathbf{u}$.
    Here, $x$ and $ux$ are realizations of $\mathbf{x}$ and $\mathbf{ux}$, and $p_\mathbf{x}$ and $p_\mathbf{ux}$ are the \textrm{pdf}s of $\mathbf{x}$ and $\mathbf{ux}$. 
\end{definition}

It is clear that, under MV, using $p_\mathbf{ux}$ to build a scoring function instead of $p_\mathbf{x}$ mitigates the effects. 
This comes as the subspaces selected by $\mathbf{u}$ are smaller in dimensionality. Hence it should take fewer samples to approximate the \textrm{pdf} of $\mathbf{ux}$.
The difficulty is that it is not yet clear how to approximate $p_\mathbf{ux}$. The following proposition elaborates on a way to do so. 
It states that by averaging a collection of marginal distributions of $\mathbf{x}$ in the subspaces given by realizations of $\mathbf{u}$, one can approximate the distribution of $p_\mathbf{ux}$.

\begin{proposition}\label{prop::statistic}
    Let $\mathbf{x}$ and $\mathbf{u}$ be as before with $p_\mathbf{x}$ myopic to the views of $\mathbf{u}$. Consider a set of independent realizations of $\mathbf{u}$: $\{u_i\}_{i=1}^{k}$. Then $\frac{1}{k} \sum_{i} p_{u_i\mathbf{x}}(u_ix)$ is a sufficient statistic for $p_{\mathbf{ux}}(ux)$.
\end{proposition}

MV appears when there is a lack of data, and its distribution is myopic. To improve OD under MV, one can exploit the distribution myopicity to model $\mathbf{x}$ in the subspaces, where less data is sufficient. Proposition \ref{prop::statistic} gives us a way to do so, by approximating $p_\mathbf{ux}$. In this way, under myopicity, this also approximates $p_\mathbf{x}$, avoiding MV. Our method, GSAAL, exploits these derivations, as we explain next.



\subsection{GSAAL}
GAAL methods tackle IA by being agnostic to outlier definition and mitigate CD through the use of multilayer neural networks~\cite{LiuYezheng-2020,LiChun-2017,poggio-2020}. GAAL methods have two steps:
\begin{enumerate}
    \item \emph{Training of the GAN}. 
    Train the GAN consisting of one generator $\mathcal{G}$ and one detector $\mathcal{D}$ using the usual $\min$-$\max$ optimization problem as in~\cite{Goodfellow-2014}.
    
    \item \emph{Training of the detector through active learning}.  
    After convergence, $\mathcal{G}$ is fixed, and $\mathcal{D}$ continues to train. 
    This last step is an active learning procedure with \cite{Zhu-2017}. 
    Following \cite{Hempstalk-2008}, $\mathcal{D}(x)$ now approximates the \textrm{pdf} of the training data $p_\mathbf{x}$. \label{enum::gaal_step2} 
\end{enumerate}
After Step~\ref{enum::gaal_step2}, the detector converges to $p_\mathbf{x}$. However, our goal is to approximate $p_{\mathbf{x}}$ by exploiting a supposed myopicity of the distribution. 
We extend GAAL methods to also address MV in what follows.
The following theorem adapts the objective function of the GAN to the subspace case and gives guarantees that the detectors converge to the marginal \textrm{pdf}s used in Proposition \ref{prop::statistic}:
\begin{theorem}\label{th::gsaal}
    Consider $\mathbf{x}$ and $\mathbf{u}$ as in the previous definition, with $x$ a realization of $\mathbf{x}$ and $\{u_i\}_i$ a set of realizations of $\mathbf{u}$.
    Consider a generator $\mathcal{G}:z\in Z\longmapsto \mathcal{G}(z) \in \mathbb{R}^d$ and $\{\mathcal{D}_i\}$, 
    $i = 1,\dots,k$, a set of detectors such as 
    $\mathcal{D}_i: u_ix \in S_i\subset \mathbb{R}^d \longmapsto \mathcal{D}_i(u_ix)\in [0,1]$.  
    $Z$ is an arbitrary noise space where $\mathcal{G}$ randomly samples from.
    Consider the following optimization problem \begin{equation}\label{eq::objgsaal}
    \begin{split}
        &\underset{\mathcal{G}}{\min} \underset{\mathcal{D}_i,~\forall i}{\max} \sum_i V(\mathcal{G},\mathcal{D}_i) =\\ 
        &\underset{\mathcal{G}}{\min} \underset{\mathcal{D}_i,~\forall i}{\max} \sum_i \mathbb{E}_{u_i\mathbf{x}}\log \mathcal{D}_i(u_ix) + \mathbb{E}_z \log\left(1 - \mathcal{D}_i\left(u_i \mathcal{G}(z)\right)\right), 
    \end{split}
    \end{equation} 
    where each addend $V(\mathcal{G},\mathcal{D}_i)$ is the binary cross entropy in each subspace.
    Under these conditions, the following holds: 
\begin{itemize}
        \item [$i)$] Each detector in optimum is $\mathcal{D}^*_i(u_ix)=\frac{1}{2}, \forall x$. Thus, in optimum $V(\mathcal{G}, \mathcal{D}_i) = -\log(4), \forall i.$\label{th::gsaal::i}
        
        \item[$ii)$] Each individual $\mathcal{D}_i$ converges to $\mathcal{D}^*_i(u_ix)=p_{u_ix}(u_ix)$ after trained in Step \ref{enum::gaal_step2} of a GAAL method.\label{th::gsaal::ii} 
        \item[$iii)$]  $\mathcal{D}^*(x) = \frac{1}{k}\sum_{i= 1}^k \mathcal{D}^*_i(u_i \mathbf{x})$ approximates $p_\mathbf{ux}(ux)$. If $p_\mathbf{x}$ is myopic, $\mathcal{D}^*(x)$ also approximates $p_\mathbf{x}(x)$.\label{th::gsaal::iii} 
    \end{itemize}
\end{theorem}
Using Theorem \ref{th::gsaal} we can extend the GAAL methods to the subspace case: 
\begin{enumerate}
    \item \emph{Training the GAN}. Train a GAN with one generator $\mathcal{G}$ and multiple detectors $\{\mathcal{D}_i\}$ with Equation (\ref{eq::objgsaal}) as the objective function. The training of each detector stops when the loss reaches its value with the optimum in Statement $(i)$.
    \item \emph{Training of the $k$ detectors by active learning}. Train each $\mathcal{D}_i$ as in Step 2 of a regular GAAL method using $\mathcal{G}$. By Statement $(ii)$ of the Theorem, each $\mathcal{D}_i$ will approximate $p_{u_i\mathbf{x}}$. By Statement $(iii)$, $\mathcal{D}(x) = \frac{1}{k}\sum_{i= 1}^k \mathcal{D}_i(u_i \mathbf{x})$ will approximate $p_\mathbf{x}$ under the myopicity of the data.

\end{enumerate} 
We call this generalization of GAAL Generative Subspace Adversarial Active Learning (GSAAL).
The appendix contains the pseudo-code for GSAAL. 

\subsection{Complexity}\label{sec::complexity}
In this section, we focus on studying the theoretical complexity of GSAAL.  
We study both its usability for training and, more importantly, for inference. 
\begin{theorem}\label{th::complexity}
    Consider our GSAAL method with generator $\mathcal{G}$ and detectors $\{\mathcal{D}_i\}_{i=1}^k$, each with four fully connected hidden layers, $\sqrt{n}$ nodes in the detectors and $d$ in the generator. Let $D$ be the training data for GSAAL, with $n$ data points and $d$ features. Then the following holds:
    \begin{itemize}
        \item[$i)$] Time complexity of training is  $\mathcal{O}(E_D\cdot n \cdot (k \cdot n + d^2))$. $E_D$  is an unknown complexity variable depicting the unique epochs to convergence for the network in dataset $D$.
        \item[$ii)$] Time complexity of single sample inference is in $\mathcal{O}(k \cdot n)$, with $k$ the number of detectors used.  
    \end{itemize}
\end{theorem}
The linear inference times make GSAAL particularly appealing in situations where the model can be trained once for each dataset, like one-class classification.
We build on this particular strength in the following section. 
\section{Experiments}\label{sec::experiments}
This section presents experiments with GSAAL. We will outline the experimental setting, and examine the handling of ``multiple views'' in GSAAL and other OD methods. We then evaluate GSAAL's performance against various OD methods and investigate its sensitivity to the number of detectors and its scalability. We also added additional experiments with other competitors outside of our related work in the appendix of the article. Our experiments used an RTX 3090 GPU and an AMD EPYC 7443p CPU running Python in Ubuntu 22.04.3 LTS. Deep neural network methods were trained on the GPU and inferred on the CPU; shallow methods used only the CPU. 


\subsection{Experimental Setting}\label{subsec::settings}
This section has three parts: First, we describe the synthetic and real data for the outlier detection experiments. Then, we describe the configuration of GSAAL. Finally, we present our competitors.  


\subsubsection{Datasets}\label{subsec::datasets}
\paragraph{Synthetic.}
We constructed synthetic datasets, each containing two correlated features, $\mathbf{x}_1$ and $\mathbf{x}_2$, along with 58 independent features $\mathbf{x}_j$, $j=3,\dots,60$ consisting of Gaussian noise. This approach simulates datasets that exhibit the MV property by integrating irrelevant features into a pair of highly correlated variables. We detail the methodology and all different correlation patterns, in the technical appendix. 


\paragraph{Real.} 
We selected 22 real-world tabular datasets for our experiments from~\cite{Songqiao-2022}. 
The selection criteria included datasets with less than 10,000 data points, more than 10 outliers, and more than 15 features, focusing on high-dimensional data while keeping the runtime (of competing OD methods) tractable. Table~\ref{tab::rwdatasets} contains the summary of the datasets. 
For datasets with multiple versions, we chose the first in alphanumeric order.
Details about each dataset are available in the original source~\cite{Songqiao-2022}. 

\begin{table}
\centering
\begin{tabular}{
>{\columncolor[HTML]{EFEFEF}}l r|
>{\columncolor[HTML]{EFEFEF}}l r}
\toprule
Dataset          & Category      & Dataset    & Category    \\ \midrule
20news    & Text          & MNIST      & Image       \\
Annthyroid       & Health  & MVTec   & Text      \\
Arrhythmia       & Cardiology    & Optdigits  & Image       \\
Cardiot.. & Cardiology    & Satellite  & Astronomy   \\
CIFAR10          & Image         & Satimage-2 & Astronomy   \\
F-MNIST     & Image         & SpamBase   & Document    \\
Fault            & Industrial    & Speech     & Linguistics \\
InternetAds      & Image         & SVHN       & Image       \\
Ionosphere      & Weather       & Waveform   & Elect. Eng. \\
Landsat          & Astronomy     & WPBC       & Oncology    \\
Letter           & Image         & Hepatitis  & Health      \\ \bottomrule
\end{tabular}
\caption{Real-world datasets used in the experiments}
\label{tab::rwdatasets}
\end{table}


\subsubsection{Network Settings}
\paragraph{Structure.} 
Unless stated otherwise, GSAAL uses the following network architecture. It consists of four fully connected layers with ReLu activation functions used in the generator and the detectors. Each layer in $k=2\sqrt{d}$ detectors has $\sqrt{n}$ nodes, where $n$ and $d$ are the number of data points and features in the training set, respectively. This configuration ensures linear inference time. The generator has $d$ nodes in each layer, a standard in GAAL approaches, which ensures polynomial training times. We assumed $\mathbf{u}$ to be distributed uniformly across all subspaces. Therefore, we obtained each subspace for the detectors by drawing uniformly from the set of all subspaces.


\paragraph{Training.} 
Like other GAAL methods \cite{LiuYezheng-2020,Zhu-2017}, we train the generator $\mathcal{G}$ together with all the detectors $\mathcal{D}_i$ until the loss of $\mathcal{G}$ stabilizes. Then we train each detector $\mathcal{D}_i$ until convergence with $\mathcal{G}$ fixed. To automate this process, we introduce an early stopping criterion: Training stops when a detector's loss approaches the theoretical optimum ($-\log(4)$), see statement $(ii)$ of Theorem \ref{th::gsaal}. For consistency across experiments, training parameters remain fixed unless otherwise noted. Specifically, the learning rates of the detectors and the generator are 0.01 and 0.001, respectively. We use minibatch gradient descent \cite{Goodfellow-2016} optimization, with a batch size of 500.


\subsubsection{Competitors}

\begin{table}
\centering
\begin{tabular}{lr}
\toprule
Type                                    & Competitors \\ \midrule
Classical                               & kNN, LOF, ABOD, SVDD \\
Subspace                                & Isolation Forest, SOD \\
Gen., uniform dist.     & Not included (see the text) \\
Gen., parametric dist.   & GMM \\
Gen., subspace behavior         & Not included (see the text) \\
GAAL                                    & MO-GAAL \\ \bottomrule
\end{tabular}
\caption{Competitors in our experiments}
\label{tab:summary_competitors}
\end{table}

We selected popular and accessible methods from each category, as summarized in Table~\ref{tab:summary_competitors}, guided by related work. We excluded generative methods with uniform distributions because they prove ineffective for large datasets~\cite{Hempstalk-2008}. We could not include a representative for generative methods with subspace behavior due to operational issues with the most relevant method in this class, \cite{Desir-2013}, caused by its outdated repository.

We used the \texttt{pyod}~\cite{pkg::pyod} library to access all competitors except MO-GAAL. We used MO-GAAL from its original source and implemented our method GSAAL in \texttt{keras}~\cite{pkg::keras}.


\subsection{Effect of Multiple Views on Outlier Detection}\label{sec::mv}
To demonstrate the effectiveness of GSAAL under MV, we use synthetic datasets. Visualizing the outlier scoring function in a 60-dimensional space is challenging, so we project it into the $\mathbf{x}_1$-$\mathbf{x}_2$ subspace. 
A method adept at handling MV should have a boundary that accurately reflects the $\mathbf{x}_1$ and $\mathbf{x}_2$ dependency structure.
The procedure is as follows: 
\begin{enumerate}
\item Generate a synthetic dataset $D^\text{synth}$ as described in section~\ref{subsec::datasets} and train the OD model.
\item Using this model, compute the scores for the points $(x_1, x_2, 0, \dots, 0)$ and visualize the level curves on the $\mathbf{x}_1$-$\mathbf{x}_2$ plane.
\end{enumerate}

Figure~\ref{fig::mv} shows results for selected datasets and competitors, which are detailed in the Appendix. It shows the level curves and decision boundaries (dashed lines) of the methods. Notably, our model effectively detects correlations in the right subspace. For example, in the \emph{banana} dataset, GSAAL accurately creates a banana-shaped boundary and outperforms other methods in distinguishing outliers from inliers in this subspace.

\begin{figure}
    \begin{subfigure}{1\linewidth}
        \centering
        \includegraphics[width = 1\linewidth]{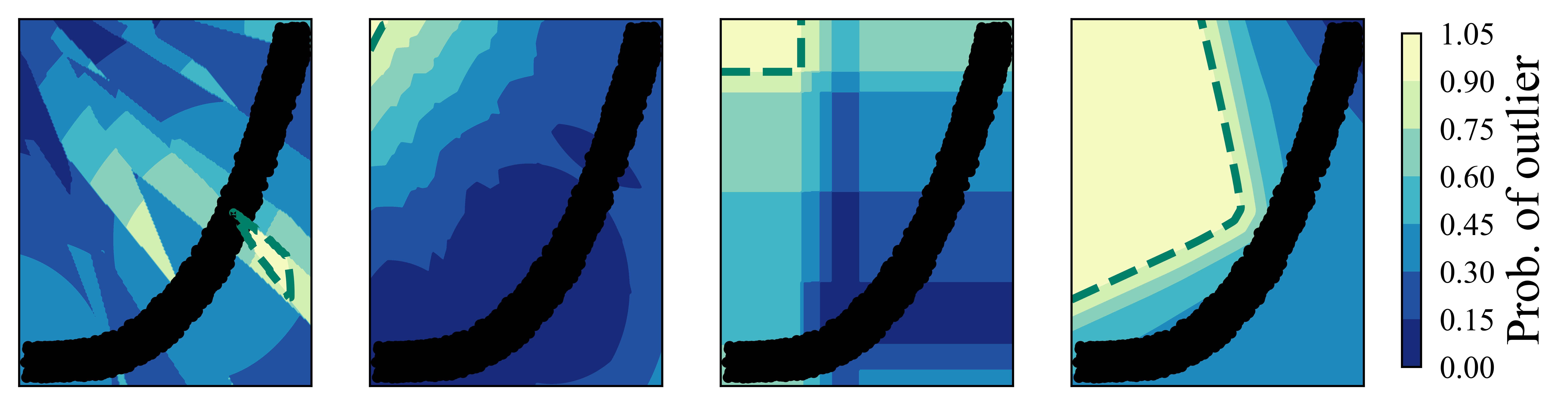}
    \end{subfigure}\\
    \begin{subfigure}{1\linewidth}
        \centering
        \includegraphics[width = 1\linewidth]{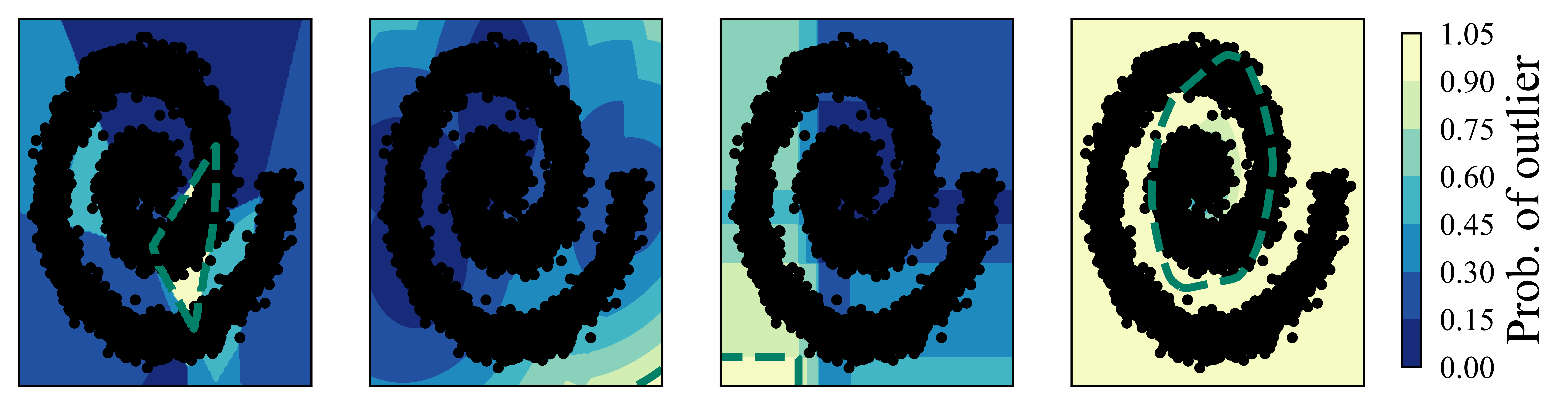}
    \end{subfigure}\\
    \begin{subfigure}{1\linewidth}
        \centering
        \includegraphics[width = 1\linewidth]{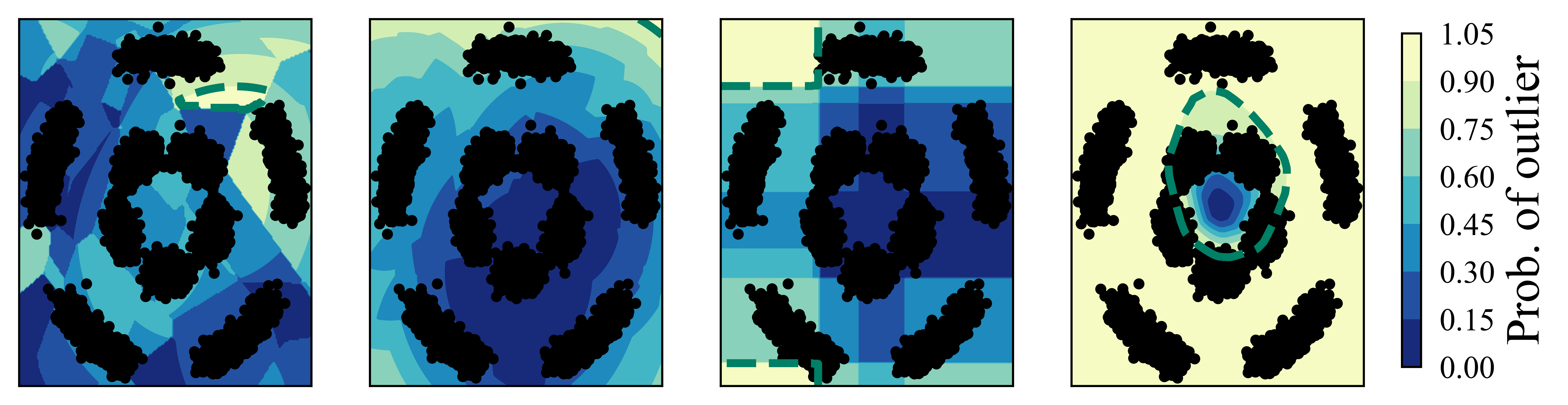}
        \caption{ABOD ~ (b) kNN ~  (c) IForest ~ (d) GSAAL}
    \end{subfigure}
    \caption{Projected classification boundaries for datasets \emph{banana}, \emph{spiral}, and \emph{star}.}
    \label{fig::mv}
\end{figure}


\subsection{One-class Classification}\label{sec::occ}
This section evaluates GSAAL on a one-class classification task \cite{seliya-2021}. First, we study the effectiveness of GSAAL on real data. Then, we explore the impact of parameter variations on the model's performance. Finally, we investigate the scalability of GSAAL in practical scenarios.

\subsubsection{Real-world Performance}
We perform the outlier detection experiments on real datasets. Specifically, we take on the task of one-class classification, where the goal is to detect outliers by training only on a collection of inliers \cite{Songqiao-2022}. 
To evaluate the performance of OD methods, we use AUC as it is resilient to test data imbalance, a common issue in OD tasks~. 
The procedure is as follows:
\begin{enumerate}
\item Split the dataset $D$ into a training set $D^\text{train}$ containing $80\%$ of the inliers from $D$, and a test set $D^\text{test}$ containing the remaining inliers and all outliers.
\item Train an outlier detection model with $D^\text{train}$ and evaluate its performance on $D^\text{test}$ with ROC AUC.
\end{enumerate}

To save space, we moved the detailed AUC results to the appendix; Figure~\ref{fig::boxplots} summarizes them. It shows that GSAAL achieves the lowest median rank. Although other subspace methods
tend to perform better with irrelevant attributes~\cite{LiuFey-2008,Kriegel-2009}, they did not outperform classical OD methods on average in our experiments. Notably, ABOD, the second best method in our experiments, performed poorly in the MV tests (Section \ref{sec::mv}). 

For statistical comparisons, we use the Conover-Iman post hoc test for pairwise comparisons between multiple populations~\cite{conover-1979}. It is superior to the Nemenyi test due to its improved type I error boundings~\cite{conover-1999}. 
Conover-Iman test requires a preliminary positive result from a multiple population comparison test, for which we employ the Kruskal-Wallis test~\cite{kruskal-1952}.

Table \ref{tab::connover} shows the test results. In each cell, `$+$' indicates that the method in the row has a significantly lower median rank than the method in the column, while `$-$' indicates a significantly higher median rank. One symbol indicates p-values $\le 0.15$ and two symbols indicate p-values $\le 0.05$. A blank indicates no significant difference. The table shows that GSAAL is superior to most of its competitors. Our method does not significantly outperform the classical methods ABOD and kNN. However, these methods struggle to detect structures in subspaces, showing their inadequacy in dealing with the MV limitation, see Section \ref{sec::mv}.

Overall, the results support GSAAL's superiority in outlier detection tasks involving multiple views. Additionally, they establish our method as the leading GAAL option for One-class classification

\begin{table*}
\centering
\begin{tabular}{c|ccccccccc}
\hline
Method  & ABOD & \textbf{GSAAL} & GMM & IForest & KNN & LOF & MO GAAL & SVDD & SOD \\ \hline
ABOD    & =    &       & ++  & ++      &     &     & ++      & ++   & ++  \\
\textbf{GSAAL}   &      & =     & ++  & ++      &     &  +   & ++      & ++    & ++  \\
GMM     & -- --   & -- --    & =   & ++      & -- --  & -- --  &         & ++   & ++  \\
IForest & -- --   & -- --    & -- --  & =       & -- --   &     & ++      &      & ++  \\
KNN     &      &       & ++  & ++       & =   &     & ++      &      & ++  \\
LOF     &      &   --    & ++  &         &     & =   & ++      &   +   & ++  \\
MO GAAL & -- --   & -- --   &     & -- --     & -- -- & -- -- & =       &      & ++  \\
SVDD    & -- --   & -- --    & -- -- &         &     &     &   --      & =    & ++  \\
SOD     & -- --   & -- --   & -- -- & -- --     & -- -- & -- -- & -- --     & -- --  & =   \\ \hline
\end{tabular}
\caption{Results of the Connover-Iman test for pairwise comparisons of the rankings.}
\label{tab::connover}
\end{table*}
\begin{figure}
    \centering
    \includegraphics[width =  \linewidth]{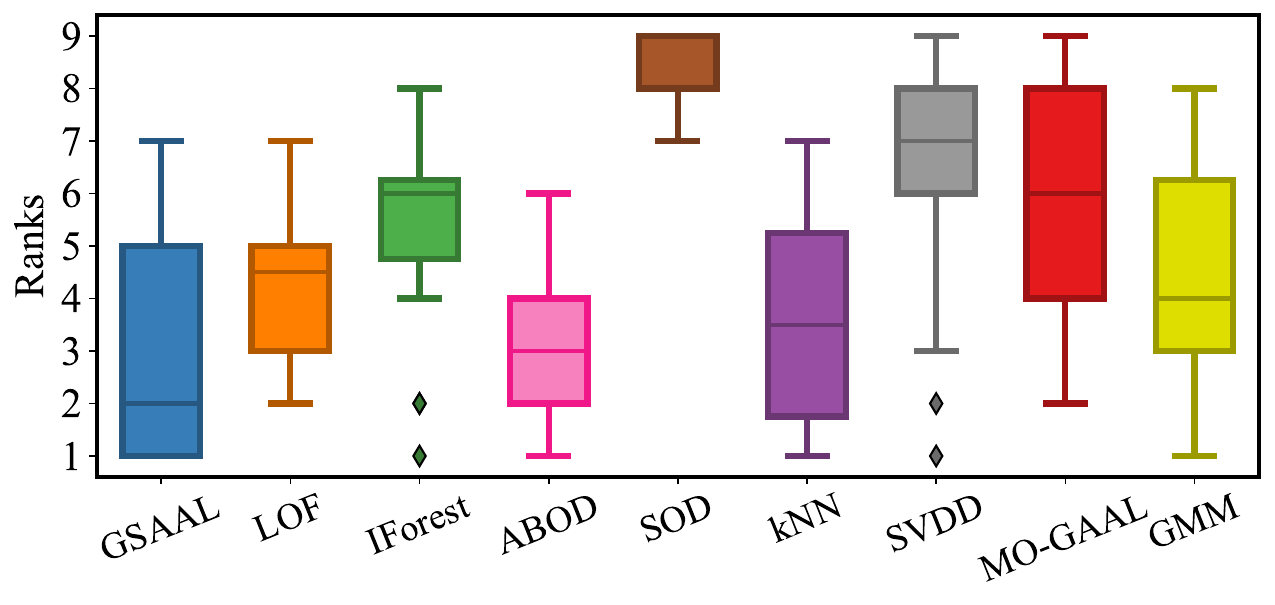}
    \caption{Boxplots of each method's rank in the real-world datasets.}
    \label{fig::boxplots}
\end{figure}

\subsubsection{Parameter Sensibility}
\begin{figure}
    \centering
    \begin{subfigure}{0.49\linewidth}
    \centering
        \includegraphics[width = 1\linewidth]{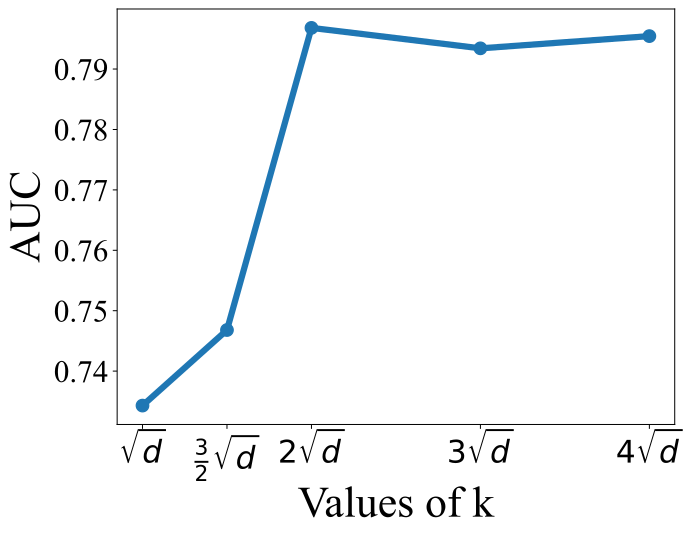}
        \caption{}
        \label{sfig::krat}
    \end{subfigure}
    \centering
    \begin{subfigure}{0.49\linewidth}
        \centering
        \includegraphics[width = 1\linewidth]{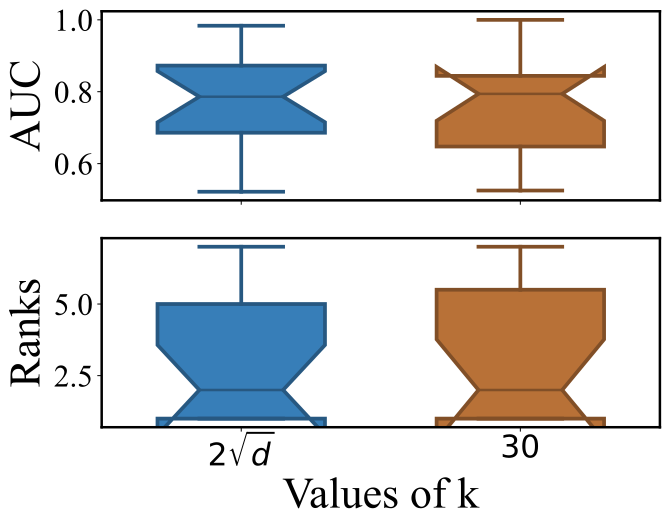}
        \caption{}
        \label{sfig::k30}
    \end{subfigure}
    \caption{Performance of the detector with different values of $k$.}
    \label{fig:enter-label}
\end{figure} 
We now explore the effect of the number of detectors in GSAAL, $k$, by repeating the previous experiments with varying $k$. Figure \ref{sfig::krat} plots the median AUC for different $k$ values, showing a stabilization at larger $k$. Next, Figure \ref{sfig::k30} compares the results with a fixed $k=30$ and the default value $k = 2\sqrt{d}$ used in the previous experiments; there is no large difference in either the AUC or the ranks. We also found that the results in Table~\ref{tab::connover} remain almost the same if one sets $k = 30$. So we recommend fixing $k = 30$, which makes GSAAL very suitable for high-dimensional data, as we will show next.


\subsubsection{Scalability}
\begin{figure}
    \centering
    \begin{subfigure}{0.49\linewidth}
    \centering
        \includegraphics[width = 1\linewidth]{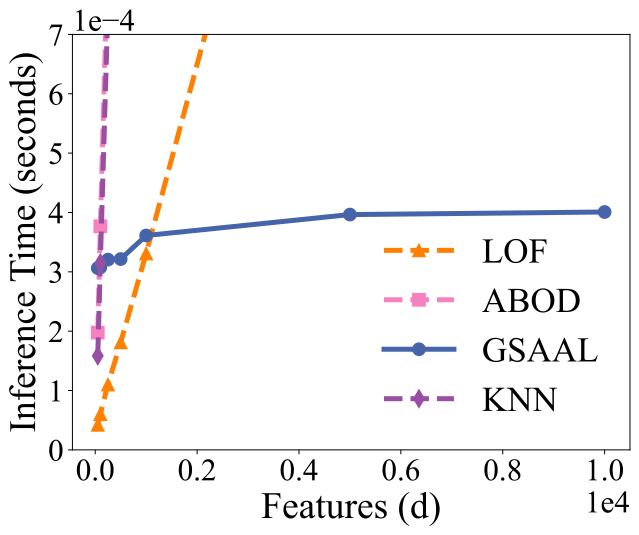} 
        \caption{}
        \label{sfig::timevsfeatures}
    \end{subfigure}
    \begin{subfigure}{0.49\linewidth}
    \centering
        \includegraphics[width = 1\linewidth]{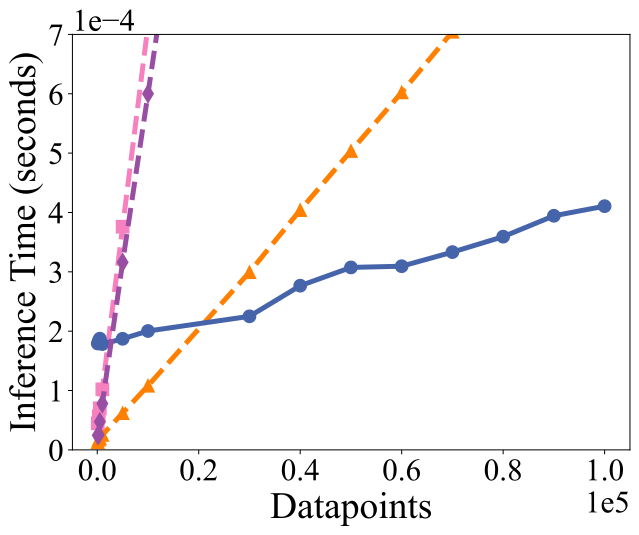}
        \caption{}
        \label{sfig::timevssamples}
    \end{subfigure}
   \caption{Plots of different performance metrics for scalability}
   \label{fig::scalability}
\end{figure} 
In section~\ref{sec::complexity}, we derived that the inference time of GSAAL scales linearly with the number of training points if the number of detectors $k$ is fixed, while it does not depend on the number of features $d$. 
This is in contrast to other methods, in particular LOF, KNN, and ABOD, which have quadratic runtimes in $d$~\cite{Breunig-2000,Kriegel-2008}. 
We now validate this experimentally. The procedure is as follows: 
\begin{enumerate}
    \item Generate datasets $D_\text{train}$ and $D_\text{test}$ consisting of random points. $|D_\text{test}|=10^6$.
    \item Train an OD method using $D_\text{train}$ and record the inference time over $D_\text{test}$. 
\end{enumerate}
Following the result of the sensitivity study, we fixed $k=30$. 
Figure~\ref{sfig::timevsfeatures} plots the inference time of a single data point as a function of the number of features when $\lvert D_{train}\rvert=500$. Figure~\ref{sfig::timevssamples} plots the inference time as a function of the number of points in $D_\text{train}$, for a fixed number of 100 features. Both figures confirm our complexity derivations and show that GSAAL is particularly well-suited for large datasets.

\section{Limitations \& Conclusions}
\label{sec::lim&conc}
We now briefly discuss future research directions and acknowledge the limitations of our study. We then summarize the main findings.

\subsection{Limitations and Future Work}

In section~\ref{sec::experiments} we randomly selected subspaces for training the detectors in GSAAL, i.e. we took a uniform distribution of $\mathbf{u}$.
This was already sufficient to demonstrate the highly competitive performance of our method.
However, GSAAL can work with any subspace search strategy to obtain the distribution of~$\mathbf{u}$, for example, the methods exploiting multiple views~\cite{Keller-2013,Muller-2012}. We have not included them in this paper due to the lack of an official implementation. In the future, we plan to benchmark various subspace search methods in GSAAL to see if there is one that consistently improves OD performance.

Next, GSAAL is limited to tabular data, since the ``multiple views'' problem has only been observed for this data type. The mathematical formulation of MV in section~\ref{sec::gsaal} does not exclude structured data. The difficulty lies in identifying good search strategies for $\mathbf{u}$ for non-tabular data, which remains an open question~\cite{gupta-2017}. 
However, depending on the type of structured data, extending GSAAL to work with it is not immediate. Therefore, building a method that exploits the theoretical derivations of GSAAL for structured data is future work.

\subsection{Conclusions}

Unsupervised outlier detection (OD) methods rely on a scoring function to distinguish inliers from outliers, since the true probability function that generated the dataset is usually unavailable in practice.
However, they face one or more of the following problems~--- Inlier Assumption (IA), Curse of Dimensionality (CD), or Multiple Views (MV). 
In this article, we have proposed the first mathematical formulation of MV, which allows for a better understanding of how to solve this occurrence.  
Using this formulation, we developed GSAAL, which is the first OD approach that solves MV, CD, and IA. In short, GSAAL is a generative adversarial network with a generator and multiple detectors fitted in the subspaces to find outliers not visible in the full space.
In our experiments on 26 different datasets, we demonstrated the usefulness of GSAAL, in particular, its ability to deal with MV and its superior performance on OD tasks with real datasets.
In addition, we have shown that GSAAL can scale up to deal with high-dimensional data, which is not the case for our most competent competitors. 
These results confirm GSAAL's ability to deal with data exhibiting MV and its usability in any practical scenario involving large datasets. 
\section{Aknowledgments}
This work was supported by the Ministry of Science, Research and the Arts Baden Württemberg, project Algorithm Engineering for the Scalability Challenge (AESC).
\appendix
\section{Theoretical Appendix}
\subsection{Previous Remarks}
Before starting to prove our main results, it is important to add a remark about our notation in this article. Whenever we denote $\mathbf{u}\mathbf{x}$, we mean the operation resulting in the following vector: 
$\mathbf{u}(\omega)\mathbf{x}(\omega)$. Thus, $\mathbf{u}\mathbf{x}$ is a random vector following its own distribution $p_\mathbf{ux}$. However, it is important to remark that $ux$, and therefore, also $u_i\mathbf{x}$, does not state the usual matrix-vector multiplication. 
What we mean by $ux$ is the operation $U\times_M x$, where $U$ stands for the range-complete version of $u$ and $\times_M$ the usual matrix multiplication. This means that whenever we write $ux$ we are considering 
\emph{the projection of $x$ into the subspace of the features selected in $u$}. 
This means that $u_i\mathbf{x}$ is the random vector composed of the features selected by $u_i$, and therefore, $p_{u_i\mathbf{x}}(u_ix)$ denotes subsequent marginal \textrm{pdf} of $\mathbf{x}$. We do not state this in the main text as it functionally does not change anything of our derivations, and simply works as a notation. The only important remarks stemming from this fact are the following:\begin{enumerate}
    \item $p_\mathbf{x}(u_ix) = p_\mathbf{x}(\pi_{u_i}(x))$, where $\pi_{u_i}$ denotes the projection of a point $x$ into the subspace of $u_i$. Therefore, we can write $p_\mathbf{x}(u_ix) = p_{u_i\mathbf{x}}(u_ix)$.
    \item The operator as stated before is not distributive. This is trivial, as given $\mathbf{u}$ a random matrix as in definition 1, ($1_d - \mathbf{u}$)$\mathbf{x}$ is defined properly, as $1_d - \mathbf{u} \in Diag(\{0,1\})$. However, $\mathbf{x} - \mathbf{u}\mathbf{x}$ denotes the vector subtraction between two vectors with different dimensionality. 
    \end{enumerate}

While not important to understand the following proofs and the derivations from the main text, understanding this is crucial for anyone seeking to work with these definitions.
\subsection{Proofs} 
We will reformulate all of the statements for completition before introducing each proof.
\begin{proposition}\label{prop::astatistic}
       Let $\mathbf{x}$ and $\mathbf{u}$ be as before with $p_\mathbf{x}$ myopic to the views of $\mathbf{u}$. Consider a set of independent realizations of $\mathbf{u}$: $\{u_i\}_{i=1}^{k}$, a realization of $\mathbf{x}$, $x$, and a realization of $\mathbf{ux}$, $ux$. Then $\frac{1}{k} \sum_{i} p_{u_i\mathbf{x}}(u_ix)$ is a sufficient statistic for $p_{\mathbf{ux}}(ux)$.
\end{proposition}
\begin{proof}
        Consider $\mathbf{x}$ and $\mathbf{u}$ as in the statement. Further consider the variable $\mathbf{w}=1_d - \mathbf{u}$, being $1_d$ the $d\times d$ identity matrix. Then, as $\mathbf{u}$ has its image in $Diag(\{0,1\})_{d\times d}$, it is clear that $        \mathbf{x} = \mathbf{u}\mathbf{x}+\mathbf{w}\mathbf{x}.$
        Therefore, we can define $p_\mathbf{x}$ as the joint \textrm{pdf} of $\mathbf{ux}$ and $\mathbf{wx}$. Now, recalling the definition of marginal distribution:\[
        p_\mathbf{ux}(ux) = \mathbb{E}_{\mathbf{wx}}(p_{\mathbf{ux|wx}}(ux|wx)).
        \] 
        Since $\mathbf{w}$ is defined in a discrete space: \[
        \mathbb{E}_{\mathbf{wx}}(p_{\mathbf{ux|wx}}(ux|wx)) = \sum_i p_{\mathbf{wx}}(w_ix)p_{\mathbf{ux|}w_i\mathbf{x}}(ux|w_ix)
        \]                                            
        We can approximate this by the sample mean, with a sample of size $k$:
        \[
        \mathbb{E}_{\mathbf{wx}}(p_{\mathbf{ux|wx}}(ux|wx)) =\frac{1}{k}\sum_{i=1}^k p_{\mathbf{ux|}w_i\mathbf{x}}(ux|w_ix) + o_{\mathbb{P}}(1).
        \]
        Now, as $\mathbf{w}$ is perfectly represented by $\mathbf{u}$, and vice versa, sampling $w_i$ is equivalent to sampling $u_i$. We will prove that, $p_{\mathbf{ux|}w_i\mathbf{x}}(ux|w_ix) = p_{u_i\mathbf{x}}(u_ix)$. After that, the rest of the proof comes clearly by substituting and recalling that the sample mean is a sufficient estimator of the expected value. First, recall the definition of conditional probability: \[
        p(ux|w_ix)=\frac{p_\mathbf{x}(ux\cap w_ix)}{p_{w_i\mathbf{x}}(w_ix)}.
        \]
        Now, by considering that $u_i = 1_d-w_i$, we have that $p(ux\cap w_ix) = 0,~\forall u\neq u_i$. Therefore, by the law of total probabilities and the definition of conditional probability: \[
         p(ux|w_ix)=\frac{p_{w_i\mathbf{x}}(w_ix)p(u_ix | w_ix)}{p_{w_i\mathbf{x}}(w_ix)}.
        \] 
        Thus:
        \[p_{\mathbf{ux}|w_i\mathbf{x}}(ux|w_ix) = p_{u_i\mathbf{x}|w_i\mathbf{x}}(u_ix|w_ix)\]
        Then, since:\[
        p_{u_i\mathbf{x}|w_i\mathbf{x}}(u_ix|w_ix) = \frac{p_\mathbf{x}(u_ix\cap w_ix)}{p_{w_i\mathbf{x}}(w_ix)},
        \]
        by myopicity, independence of $\{u_i\}_i$, and considering that sampling from $\mathbf{u}$ is the same as sampling from $\mathbf{w}$, it trivially follows that:
        \[
        p_{u_i\mathbf{x}|w_i\mathbf{x}}(u_ix|w_ix) = p_{u_i\mathbf{x}}(u_ix).  
        \]
        We can retrieve the equality in the statement by consdering $\mathbf{u}$ to be uniformly distributed~---as we do in section \ref{sec::experiments}. 
\end{proof}

    \begin{theorem}\label{th::agsaal}
    Consider $\mathbf{x}$ and $\mathbf{u}$ as in the previous definition, with $x$ a realization of $\mathbf{x}$ and $\{u_i\}_i$ a set of realizations of $\mathbf{u}$.
    Consider a generator $\mathcal{G}:z\in Z\longmapsto \mathcal{G}(z) \in \mathbb{R}^d$ and $\{\mathcal{D}_i\}$, 
    $i = 1,\dots,k$, a set of detectors such as 
    $\mathcal{D}_i: u_ix \in S_i\subset \mathbb{R}^d \longmapsto \mathcal{D}_i(u_ix)\in [0,1]$.  
    $Z$ is an arbitrary noise space where $\mathcal{G}$ randomly samples from.
    Consider the following objective function \begin{equation}\label{eq::aobjgsaal}
    \begin{split}
        &\underset{\mathcal{G}}{\min} \underset{\mathcal{D}_i,~\forall i}{\max} \sum_i V(\mathcal{G},\mathcal{D}_i) =\\ 
        &\underset{\mathcal{G}}{\min} \underset{\mathcal{D}_i,~\forall i}{\max} \sum_i \mathbb{E}_{u_i\mathbf{x}}\log \mathcal{D}_i(u_ix) + \mathbb{E}_z \log\left(1 - \mathcal{D}_i\left(u_i \mathcal{G}(z)\right)\right) 
    \end{split}
    \end{equation}
    Under these conditions, the following holds: 
\begin{itemize}
        \item [$i)$] Each detector's loss in optimum is $V(\mathcal{G},\mathcal{D}^*_i)=\frac{1}{2}$.\label{th::agsaal::i}
        \item[$ii)$] Each individual $\mathcal{D}_i$ converges to $\mathcal{D}^*_i(u_ix)=p_{u_ix}(u_ix)$ after trained in Step \ref{enum::gaal_step2} of a GAAL method.\label{th::agsaal::ii} 
        \item[$iii)$]  $\mathcal{D}^*(x) = \frac{1}{k}\sum_{i= 1}^k \mathcal{D}^*_i(u_i \mathbf{x})$ approximates $p_\mathbf{ux}(ux)$. If $p_\mathbf{x}$ is myopic, $\mathcal{D}^*(x)$ also approximates $p_\mathbf{x}(x)$.\label{th::agsaal::iii} 
    \end{itemize}
\end{theorem}

\begin{proof}
    This proof will follow mainly the results in \cite{Goodfellow-2014}, adapted for our case. We will first derivative two general results that we are going to use to immediately prove $(i),(ii)$ and $(iii)$. First, consider the objective function \begin{equation*}
    \begin{split}
             \sum_i V(\mathcal{G},\mathcal{D}_i) = \sum_i &\mathbb{E}_{u_i\mathbf{x}\sim p_{u_i\mathbf{x}}}\log (\mathcal{D}_i(u_ix)) + \\ &\mathbb{E}_{\mathbf{z}\sim p_\mathbf{z}} (1-\log(\mathcal{D}_i(u_i\mathcal{G}(z)))),
    \end{split}
    \end{equation*}
    where $\mathbf{z}$ is the random vector used by $\mathcal{G}$ to sample from the noise space $Z$.
    We will write $\mathbb{E}_\mathbf{x}, \mathbb{E}_\mathbf{z}$ and $\mathbb{E}_{u_i\mathbf{x}}$ instead of $\mathbb{E}_{\mathbf{x}\sim p_\mathbf{x}}, \mathbb{E}_{\mathbf{z}\sim p_\mathbf{z}}$ and $\mathbb{E}_{u_i\mathbf{x}\sim p_{u_i\mathbf{x}}}$ as an abuse of notation. 
    
    The problem is, then, to optimize: \begin{equation}\label{proof::objective}
        \underset{\mathcal{G}}{\min} \underset{\mathcal{D}_i,~\forall i}{\max} \sum_i V(\mathcal{G},\mathcal{D}_i).
    \end{equation}
    Fixing $\mathcal{G}$ and maximizing for all $\mathcal{D}_i$, each detector individually maximizes $V(\mathcal{G},\mathcal{D}_i)$. Let us try to obtain the optimal of each $\mathcal{D}_i$ with a fixed $\mathcal{G}$. First, we write:
    \[
    \begin{split}
    V(\mathcal{G},\mathcal{D}_i) = &\int_{u_ix}p_{u_i\mathbf{x}}(u_ix)\log \mathcal{D}_i(u_ix) du_ix + \\
    &\int_{z} p_\mathbf{z}(z) \log(1-\mathcal{D}_i(u_i\mathcal{G}(z))) dz.
    \end{split}
    \]
    As $\mathcal{G}$ uses $\mathbf{z}$ to sample from its sample distribution $p_\mathcal{G}(x)$, we can rewrite the second addent, like in \cite{Goodfellow-2014}, as:\[
    \begin{split}
     V(\mathcal{G},\mathcal{D}_i) = &\int_{u_ix}p_{u_i\mathbf{x}}(u_ix)\log \mathcal{D}_i(u_ix) du_ix + \\
    &\int_{u_ix} p_\mathcal{G}(u_ix) \log(1-\mathcal{D}_i(u_ix)) du_ix.
    \end{split}
    \]
    Aggregating both integrals, we have a function of the type $f(y) = a\log(y) + b\log(1-y)$, with $a,b\in \mathbb{R}-\{0\}$. It is a known fact in calculus that $f(y)$ obtains its optimum in $y = \frac{a}{a+b}$. As $f(y)\in \mathbb{R}^+$, $V(\mathcal{G},\mathcal{D}_i)$ obtains its optimum for a given $\mathcal{G}$ in: \begin{equation}\label{proof::D^*}
    D^*_i(u_ix) = \frac{p_{u_i\mathbf{x}}(u_ix)}{p_{u_i\mathbf{x}}(u_ix) + p_\mathcal{G}(u_ix)}.
    \end{equation}
    
    Let us now consider the following function\begin{equation}
        \begin{split}
            C(\mathcal{G}) &= \sum_i \underset{\mathcal{D}_i,~\forall i}{\max} V(\mathcal{G},\mathcal{D}_i)\\
            &=\sum_i \mathbb{E}_{u_i\mathbf{x}}\log\frac{p_{u_i\mathbf{x}}(u_ix)}{p_{u_i\mathbf{x}}(u_ix) + p_{\mathcal{G}}(u_ix)} +\\& \mathbb{E}_{u_i\mathbf{x}\sim p_\mathcal{G}}\log\frac{p_\mathcal{G}(u_ix)}{p_{u_i\mathbf{x}}(u_ix) + p_{\mathcal{G}}(u_ix)}.
        \end{split} 
    \end{equation}
    This is known in Game Theory as the cost function of player ``$\mathcal{G}$'' in the null-sum game defined by the $\min \max$ optimization problem. \cite{Goodfellow-2014} refers to it as the virtual training criterion of the GAN. The adversarial game defined by (\ref{proof::objective}) reaches an equilibrium (and thus, the $\min\max$ problem an optimum) whenever $C(\mathcal{G})$ is minimized. We will study the value of $\mathcal{G}$ in such equilibrium and use it, together with (\ref{proof::D^*}), to prove the statements. 

    Rewriting $C(\mathcal{G})$ it is clear that: \begin{equation*}
        \begin{split}
            C(\mathcal{G}) = &\sum_i KL\left(p_{u_i\mathbf{x}(u_ix)}\| \frac{p_{u_i\mathbf{x}}(u_ix) + p_\mathcal{G}(u_ix)}{2}\right)\\
        &+ KL\left(p_{\mathcal{G}}(u_ix)\| \frac{p_{u_i\mathbf{x}}(u_ix) + p_\mathcal{G}(u_ix)}{2}\right).
        \end{split}
    \end{equation*}
    This expression corresponds to that of a sum of multiple binary cross entropies between a population coming from $p_{u_i\mathbf{x}}$ and from $p_\mathcal{G}$ projected by $u_i$. Therefore, as we know, we can rewrite: \[
    C(G) = \sum_i 2JSD(p_{u_i\mathbf{x}(u_ix)}\| p_\mathcal{G}(u_ix)),
    \]
    with $JSD$ the Jensen-Shannon divergence. Since $JSD(s\|r)\in [0,\log(2))$, it is clear that $C(\mathcal{G})$ obtains its minimum only whenever \begin{equation}\label{proof::G^*}
        p_\mathcal{G}(u_ix) = p_{u_i\mathbf{x}}(u_ix), \forall\forall x\footnote{For almost all $x$};
    \end{equation}
    and for all $i\in \{1,\dots, k\}$.

    Knowing $\mathcal{G}$ and $\mathcal{D}_i$ in the optimum for all $i$, we can prove the statements above:
    \paragraph{(i)} As $p_\mathcal{G}(u_ix)= p_{u_i\mathbf{x}}(u_ix)$ for almost all $x$, in the optimum of (\ref{proof::objective}), it is immediate that: \[
        \mathcal{D}_i(u_ix) = \frac{1}{2},
    \]
    i.e., the detectors cannot differentiate between the real training data and the synthetic data of the generator. If one employs the numerically stable version of each $V(\mathcal{G}, \mathcal{D}_i)$ (equivalent as the numerically stable version of the binary cross entropy \cite{pkg::keras}), it is trivial to see that \[
    V^{\text{stable}}(\mathcal{G}, \mathcal{D}_i) = \log(2).
    \]

    \paragraph{(ii)} After optimizing (\ref{proof::objective}), training each $D_i$ individually with $\mathcal{G}$ fixed, is the equivalent of building a two-class classifier distinguishing between the artificial class generated by $p_\mathcal{G}(u_ix)=p_{u_i\mathbf{x}}(u_ix)$ and the real data coming from $p_{u_i\mathbf{x}}(u_ix)$. By \cite{Hempstalk-2008}, the resulting two-class classifier would be such as:\[
    D_i(u_ix) = p_{u_i\mathbf{x}}(u_ix).
    \]
    \paragraph{(iii)} By proposition \ref{prop::astatistic} and statement $(ii)$, $\frac{1}{k}\sum_i D^*_i(u_ix)$ is a sufficient estimator for $p_{u_i\mathbf{x}}(u_ix)$. By myopicity, it is also of $p_\mathbf{x}(x)$.
    \end{proof}

\begin{theorem}\label{th::acomplexity}
    Giving our GSAAL method with generator $\mathcal{G}$ and detectors $\{\mathcal{D}_i\}_{i=1}^k$, each with four fully connected hidden layers, $\sqrt{n}$ nodes in the detectors and $d$ in the generator, we obtain that:
    \begin{itemize}
        \item[$i)$] The training time complexity is bounded with $\mathcal{O}(E_D\cdot n \cdot (k \cdot n + d^2))$, for a dataset $D$ with $n$ training samples and $d$ features. $E_D$  is an unknown complexity variable depicting the unique epochs to convergence for the network in dataset $D$.
        \item[$ii)$] The single sample inference time complexity is bounded with $\mathcal{O}(k \cdot n)$, with $k$ the number of detectors used.  
    \end{itemize}
\end{theorem}
\begin{proof}
    An evaluation of a neural network is composed of two steps, the backpropagation, and the fowardpass steps. While training the network requires both, inference requires only a fowardpass. Therefore, we will first prove $(ii)$ and will build upon it to prove $(i)$.
    \paragraph{(ii).} GSAAL consists of a generator and $k$ detectors. Single point inference consists of a single fowardpass of all the detectors. We will first prove the general complexity of a fowardpass of a general fully connected 4 layer network and will use it to derive all the other complexities. Let us consider three weight matrices $W_{ji}$, $W_{hj}$ and $W_{lh}$ each between two layers, with $j,i,h$ and $l$ being the number of nodes in each. Therefore, $W_{ji}$ denotes a matrix with $j$ rows and $i$ columns, and so on. Now, let us consider $x_{i1}$ the datapoint after passing the input layer. Lastly, without any loss of generality, consider $f$ to be the activation function for all layers. This way, the forward pass of a single detector can be written as: \[
    c_{l1}=f\left(W_{lh}f\left(W_{hj}f\left(W_{ji}x_{i1}\right)\right)\right).
    \]
    We will study the complexity in the first layer and use it to derive the complexity of the others. $A_{j1} = W_{ji}x_{i1}$ is a simple matrix-vector multiplication that we know to be $\mathcal{O}(j\cdot i)$ atmost. Then, as $f$ is an activation function, $f(A_{j1})$ is equivalent to writing $f_{j1} \odot A_{j1}$, with $\odot$ being the element-wise multiplication. Thus, $f\left(W_{ji}x_{i1}\right)$ is: \[
    \mathcal{O}(j\cdot i + j) = \mathcal{O}(j\cdot (i + 1)) = \mathcal{O}(j\cdot i).
    \]
    Doing this for all layers, we obtain: \begin{equation}\label{eq::acomplexityfp}
    \mathcal{O}(l\cdot h + k \cdot j + j\cdot i).
    \end{equation}
    As all layers have $\sqrt{n}$ nodes, 
    \[\mathcal{O}(3n) = \mathcal{O}(n).\]
    As we have $k$ detectors, the complexity for a fowardpass of all detectors, and thus, for a single sample inference of GSAAL is: \[
    \mathcal{O}(k\cdot n).
    \]
    \paragraph{(i).} A backpropagation step has the same complexity as an inference step on all training samples. As we have $n$ training samples, this then becomes \[
    \mathcal{O}(k\cdot n^2)
    \]
    for the detectors. As the training consists of multiple epochs, we will write
    \[
    \mathcal{O}(E_D\cdot k \cdot n^2),
    \]
    with $E_D$ being the number of epochs needed for convergence for the training data set $D$. 
    As the training consists of both backpropagation and fowardpass steps on all training samples, the total training time complexity for all detectors is: \[
    \mathcal{O}(E_D\cdot k \cdot n^2 + k\cdot n^2) = \mathcal{O}(E_D\cdot k \cdot n^2).
    \]
    As we also need to consider the generator, we will use equation \ref{eq::acomplexityfp} to derive both steps on the generator. As the generator is also a fully connected 4-layer network, with all layers having $d$ nodes, the complexity for a single fowardpass is: \[
    \mathcal{O}(d^2).
    \]
    As during training one generates $n$ samples during each fowardpass: \[
    \mathcal{O}(n \cdot d^2).
    \]
    Now, on each backpropagation pass the network calculates the backpropagation error for each generated sample, thus, \[
    \mathcal{O}(n \cdot d^2)
    \]
    is also the time complexity for the backpropagation step of the generator. Considering all $E_D$ epochs and both backpropagation and fowardpass steps of the generator and all the detectors, the time complexity of GSAAL's training is: \[
    \mathcal{O}(E_D\cdot k \cdot n^2 + E_D \cdot n \cdot d^2) =  \mathcal{O}(E_D \cdot n \cdot (k\cdot n + d^2))
    \]
\end{proof}
\subsection{Multiple Views (extension)}
In this section we extend the derivations in section 3.1 by providing an example of a myopic distribution: \begin{example}[Myopic distribution]
Consider a $\mathbf{x}$ like in example 1. Here, it is clear that $\mathbf{x_1,x_2}\bot \mathbf{x_3}$.
Consider, then, $\mathbf{u}$ such that: \[
\mathbf{u}:\{1\} \longrightarrow \{diag(1,1,0)\}.   
\] To test whether $p_\mathbf{x}$ is myopic, we employed a simple test utilizing a statistical distance ($MMD$ with the identity kernel) between $p_\mathbf{x}$ and $p_\mathbf{ux}$. This way, if $\hat{MMD}(p_\mathbf{x}\| p_\mathbf{ux}) = 0$, it would be clear that the equality holds. As a control measure, we also calculated the same distance for a different population $\mathbf{x}'$, where $\mathbf{x}_3 = \mathbf{x}_1^2$. We have plotted the results in image \ref{fig::ammd}, where Population 1 refers to $\mathbf{x}$ and Population 2 to $\mathbf{x}'$. As we can see, we do obtain a positive result in the test of myopicity for $\mathbf{x}$ and a negative one for $\mathbf{x}'$.
\end{example}
\begin{figure}
    \centering
    \includegraphics[width = \linewidth]{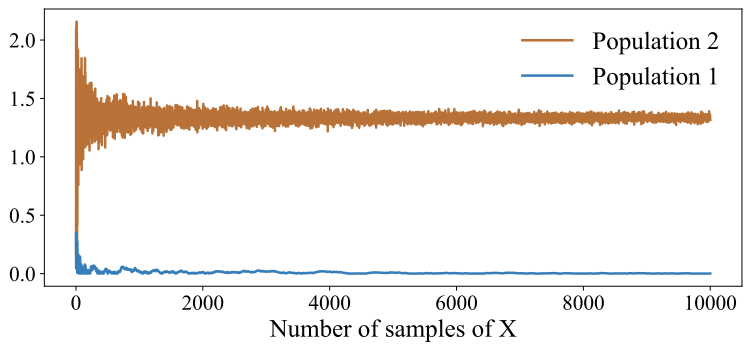}
    \caption{Difference in statistical distance between two populations.}
    \label{fig::ammd}
\end{figure}
\subsection{GSAAL (extension)}
We now extend the results from section 3.2 by providing the pseudocode for the training of our method. It is important to consider that, while theorem \ref{th::agsaal} formulates the optimization problem in terms of the neural networks $\mathcal{G}$ and $\{\mathcal{D}_i\}_i$, in practice this will not be the case. Instead, we will consider the optimization in terms of their weights, $\Theta_\mathcal{G}$ and $\Theta_{\mathcal{D}_i}$. Therefore, in practice, the convergence into an equilibrium will be limited by the capacity of the networks themselves \cite{Goodfellow-2016}. We considered the optimization to follow minibatch-stochastic gradient descent \cite{Goodfellow-2016}. To consider any other minibatch-gradient method it will suffice to perform the necessary transformations to the gradients. 
\begin{algorithm}
\caption{GSAAL training}\label{alg::gsaal_training}
\begin{algorithmic}[1]
\REQUIRE{Data set $D$, Number of Discriminators $\kappa$, $\mathbf{u}$, $epochs$, $stop\_epoch$}
\STATE Initialize Generator \(\mathcal{G}\)
\COMMENT{\#$d$ is the dimensionality of $D$}
\STATE \(\{u_i\}_{i=1}^\kappa \gets\) \textsc{DrawFrom$\mathbf{u}$}($\kappa$) 
\STATE Initialize Discriminators $\{\mathcal{D}_i\}_{i=1}^\kappa$ with unique subspaces $\{u_i\}_{i=1}^\kappa$
\FOR{\(epoch \in \{1,...,epochs\}\)}
    \FOR{\(batch \in \{1,...,batches\}\)}
        \STATE\(noise \gets\) Random noise \(z^{(1)},...,z^{(m)}\) from $Z$
        \STATE \(data \gets\) Draw current batch \(x^{(1)},...,x^{(m)}\)
        \FOR{\(j \in \{1...k\}\)}
           \STATE Update \(\mathcal{D}_j\) by ascending the stochastic gradient: $\nabla_{\Theta_{\mathcal{D}_j}} \frac{1}{m}\sum_{i=1}^m \log(\mathcal{D}_j(u_jx^{(i)}))+\log(1-\mathcal{D}_j(u_j\mathcal{G}(z^{(i)})))$
        \ENDFOR 
        \IF{$epoch < stop\_epoch$}
            \STATE Update \(\mathcal{G}\) by descending the stochastic gradient: $\nabla_{\Theta_G}\frac{1}{k}\sum_{j=1}^k \frac{1}{m}\sum_{i=1}^m \log(1- \mathcal{D}_j(\mathcal{G}(z^{(i)})))$
        \ENDIF
    \ENDFOR
\ENDFOR
\end{algorithmic}
\end{algorithm}

The pseudocode is located in Algorithm \ref{alg::gsaal_training}. As it is the training for the method, it takes both the parameters for the method and the training. In this case, $epochs$ refers to the total number of epochs we will train in total, while $stop\_epoch$ marks the epoch where we start step 2 of the GAAL training. 
Lines 1-3 initialize both the detectors in their subspaces and the generator with random weight matrices $\Theta_{\mathcal{D}_i}$ and $\Theta_{\mathcal{G}}$. Lines 4-13 correspond to the normal GAN training loop across multiple epochs, which we referred to as step 1 of a GAAL method in the main text, if $epoch < stop\_epoch$. Here we proceed with training each detector and the generator using their gradients. Lines 8-10 update each detector by ascending its stochastic gradient, while line 11 updates the generator by descending its stochastic gradient. After the normal GAN training, we start the active learning loop \cite{LiuYezheng-2020} once $epoch \geq stop\_epoch$. The only difference with the regular GAN training is that $\mathcal{G}$ remains fixed, i.e., we do not descend using its gradient. This allows us to additionally train the detectors and, in case of equilibrium of step 1, converge to the desired marginal distributions as derived in theorem \ref{th::agsaal}.

\section{Experimental Appendix}
In this section, we will include a supplementary experiment testing the IA condition for competition, and an extension of both experimental studies featured in the main text. All of the code for the extra experiment, as well as for all experiments in the main text, can be found in our remote repository\footnote{https://github.com/WamboDNS/GSAAL}.

\subsection{Effects of Inlier Assumptions on Outlier Detection}
\begin{table*}
\centering
\begin{tabular}{llll}
\toprule 
Outlier Type & Assumption Description                                                                                                   & Outlier Description                                                                                 & $M$                         \\ \midrule \rowcolor[HTML]{EFEFEF} 
Local        & \begin{tabular}[c]{@{}l@{}}Assumes that all inliers are\\  located close to other inliers\end{tabular}                   & \begin{tabular}[c]{@{}l@{}}As a result, outliers are \\ far away from inliers\end{tabular}          & LOF                         \\ 
Angle        & \begin{tabular}[c]{@{}l@{}}Assumes that all inliers \\ have other inliers in all angles from their position\end{tabular} & \begin{tabular}[c]{@{}l@{}}As a result, outliers are \\ not surrounded by other points\end{tabular} & ABOD                        \\ \rowcolor[HTML]{EFEFEF} 
Cluster      & \begin{tabular}[c]{@{}l@{}}Assumes that all inliers\\  form large clusters of data\end{tabular}                          & \begin{tabular}[c]{@{}l@{}}As a result, outliers are \\ gathered in small clusters\end{tabular}     & $F_{n,\mu + \varepsilon_i}$ \\  \bottomrule
\end{tabular}
\caption{Different outliers generated for the experiments.}
\label{tab::IA}
\end{table*}
GAAL methodologies are capable of dealing with the inlier assumption by learning the correct inlier distribution $p_\mathbf{x}$ without any assumption \cite{LiuYezheng-2020}. While this should also extend to our methodology, we will study experimentally whether this condition holds in practice. To do so, as one cannot identify beforehand whether a method is going to fail due to IA, we will generate synthetic datasets. This will allow us to generate outliers that we know to follow from a specific IA, ensuring that failure comes from the anomalies themselves. We will include all of the code in the code repository.
To generate the synthetic datasets we follow: \begin{enumerate}
    \item Generate $D$, a population of $2000$ inliers following some distribution $F$ in $\mathbb{R}^{20}$.
    \item Select an outlier detection method $M$ with some assumption about the normality of the data and fit it using $D$. We will call such $M$ as the reference model for the generation. \label{enum::M}
    \item Generate $400$ outliers by sampling on $\mathbb{R}^{20}$ uniformly and keeping only those points $o$ such that $M(o) = 1$ (i.e., they are detected as outliers). We will write $O^D$ to refer to such a collection of points. \label{enum::gen.out}
    \item Repeat step \ref{enum::gen.out} $10$ times, to obtain $O^D_1,\dots, O^D_{10}$.
    \item Sample out $20\%$ of the points in $D$. The remainder $80\%$ will be stored in $D^\text{train}$, and the other $20\%$ in $D^\text{test}_1,\dots, D^\text{test}_{10}$ together with each $O^D_{i}$.
\end{enumerate}
These steps were repeated $4$ times with different $F$, to create $4$ different training sets and $40$ different testing sets, corresponding to a total of $40$ different datasets employed per model $M$ selected in step \ref{enum::M}. As we used $3$ different reference models, we have a total of $120$ different datasets employed in this experiment alone. In particular, the models used for this are collected in table \ref{tab::IA}. The table contains the name of the outlier type, the description of the IA taken to generate them, and a brief description of how the outliers should look. Column $M$ contains the method employed to generate each, these being $LOF$, $ABOD$, and the same inlier distribution as $D$, but with multiple shifted means $\mu_i$ and with a significantly lower amount of points $n$. A visualization of how these outliers would look with $2$ features is located in figure \ref{fig::aIA_2dex}.
\begin{figure}
\centering
    \begin{subfigure}{0.3\linewidth}
    \centering
    \includegraphics[width = \linewidth]{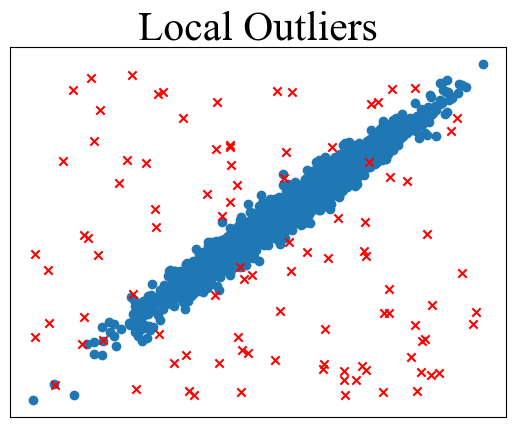}
    \caption{}
    \label{sfig::local_2d}
    \end{subfigure}
    \begin{subfigure}{0.3\linewidth}
    \centering
    \includegraphics[width = \linewidth]{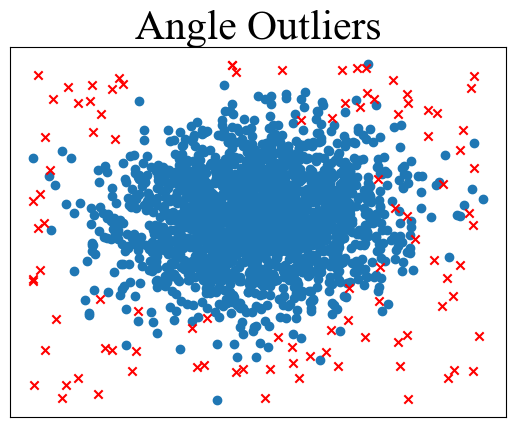}
    \caption{}
    \label{sfig::angle_2d}
    \end{subfigure}
    \begin{subfigure}{0.3\linewidth}
    \centering
    \includegraphics[width = \linewidth]{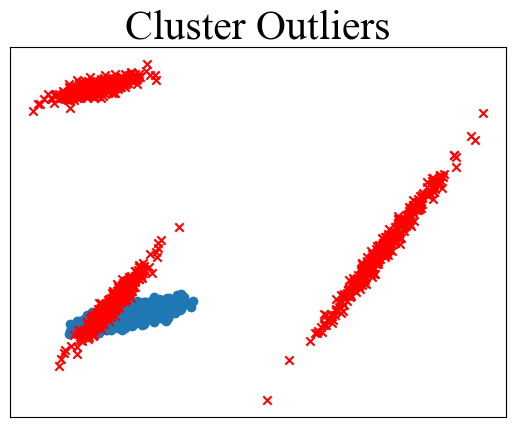}
    \caption{}
    \label{sfig::cluster_2d}
    \end{subfigure}
    \caption{2D-example of the different types of anomalies we generate using the method summarized in table \ref{tab::IA}.}
    \label{fig::aIA_2dex}
\end{figure}

To study how different methods behave when detecting these outliers, we have performed the same experiments as in section 4.3, but with these synthetic datasets. Figure \ref{fig::aIA_boxplots} gathers all the AUCs of a method in $3$ boxplots, one for each outlier type in each training set. Additionally, we grouped all based on the IA and assigned a similar color for all of them. 
We have done this for the classical OD methods LOF, ABOD, and kNN, besides our method GSAAL. We cropped the image below $0.45$ in the $y$ axis as we are not interested in results below a random classifier. As we can see, classical methods seem to correctly detect outliers for an outlier type that verifies its IA. However, whenever we introduce outliers behaving outside of their IA, the performance hit is significant. Notoriously, it appears that none of them had trouble detecting the \emph{Local} and \emph{Angle} outlier type. regardless of their IA. This can be easily explained by those outliers types being similar, as we can see in figure \ref{fig::aIA_2dex}. 
On the other hand, GSAAL manages to have a significant detection rate regardless of the outlier type. 

\begin{figure*}
    \centering
    \includegraphics[width = \linewidth]{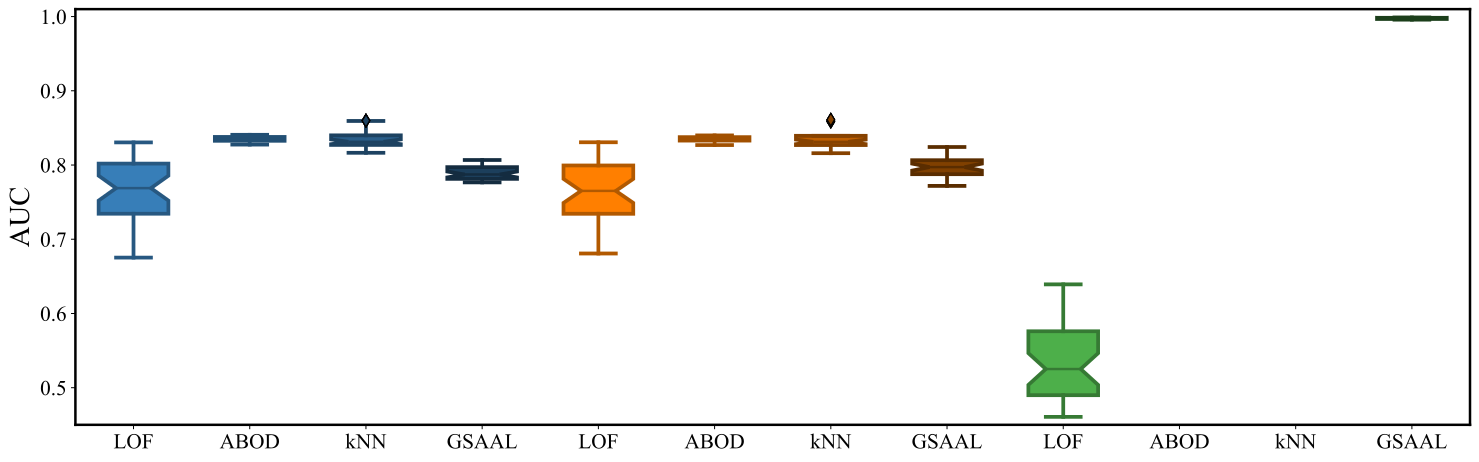}
    \caption{AUCs of the different methods in the IA experiments. From left to right: Local (blue), Angle (orange) and Cluster (green).}
    \label{fig::aIA_boxplots}
\end{figure*}

\subsection{Effects of Multiple Views on Outlier Detection (extension)}

In this section, we will include a brief description of the generation process for the datasets used in section 4.2. We will also perform the same experiment as in section 4.2 for all methods showcased in the main text and additional datasets.
The datasets were generated by the following formulas:
\begin{itemize}
\item\emph{Banana.} Given $\theta \in [0,\pi]$ we have $\mathbf{x} = \sin(\theta) + U(0,0.1)$ and $\mathbf{y} = \sin(\theta)^3 + U(0,0.1)$.
\item\emph{Spiral.} Given $\theta \in [0,4\pi]$ and $r\in(0,1)$, we have $\mathbf{x}=r\cos(\theta) + U(0,0.1)$ and $\mathbf{y}= r\sin(\theta)$.
\item\emph{Star.} Given $\theta \in [0,2\pi]$ and $r \in \left\{r\in \mathbb{R}| r = \sin(5\theta); r \geq 0,1,0.4\right\},$ we have $\mathbf{x} = r\cos(\theta) + U(0,0.1)$ and $\mathbf{y} = r\sin(\theta) + U(0,0.1)$.
\item\emph{Circle.} Given $\theta \in [0,2\pi]$, we have $\mathbf{x} = \cos(\theta) + U(0,0.1)$ and $\mathbf{y} = \sin(\theta) + U(0,0.1).$
\item\emph{L.} Given $x_1 = N(0,0.1), x_2 = U(0,5), y_1 = U(-5,0), \text{and } y_2 = N(0,0.1)$; we have $\mathbf{x} = \texttt{concat}(x_1,x_2)$ and $\mathbf{y} = \texttt{concat}(y_1,y_2).$
\end{itemize}
We considered $N(0,0.1)$ to denote a random normal realization with $\mu = 0$ and $\sigma^2=0.1$, and $U(a,b)$ to denote a uniform realization in the $[a,b]$ interval.

Figure \ref{fig::aMV_appendix} contains all images from the MV experiment. We do not have any new insight beyond the ones exposed in the main article. Note that we have included all methods but SOD. The reason was that SOD failed to execute for datasets Star, Spiral, and Circle. 
\begin{figure*}
\begin{subfigure}{1\linewidth}
    \centering
    \includegraphics[width = 1\linewidth]{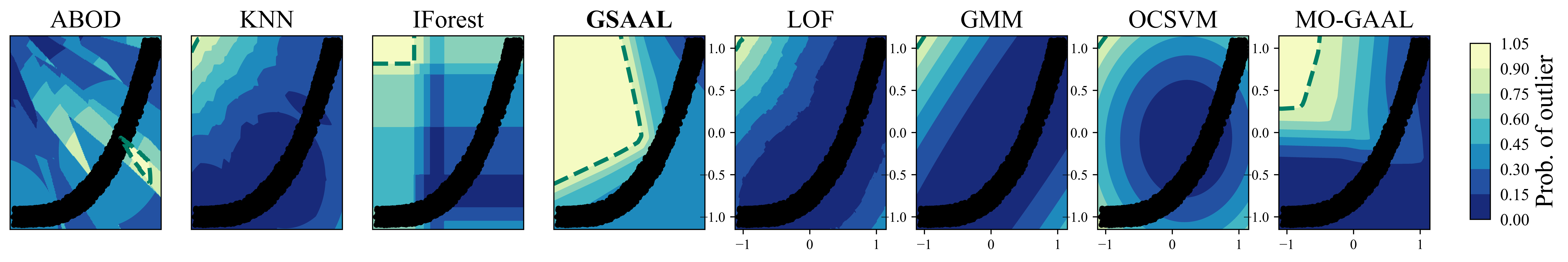}
    \caption{Banana}
    \label{fig:aenter-label}
\end{subfigure}
\begin{subfigure}{1\linewidth}
    \centering
    \includegraphics[width = 1\linewidth]{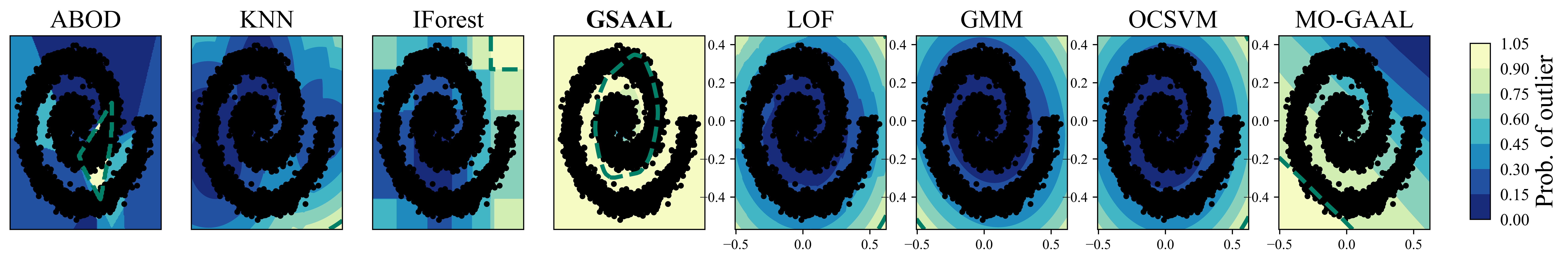}
    \caption{Spiral}
    \label{fig:benter-label}
\end{subfigure}
\begin{subfigure}{1\linewidth}
    \centering
    \includegraphics[width = 1\linewidth]{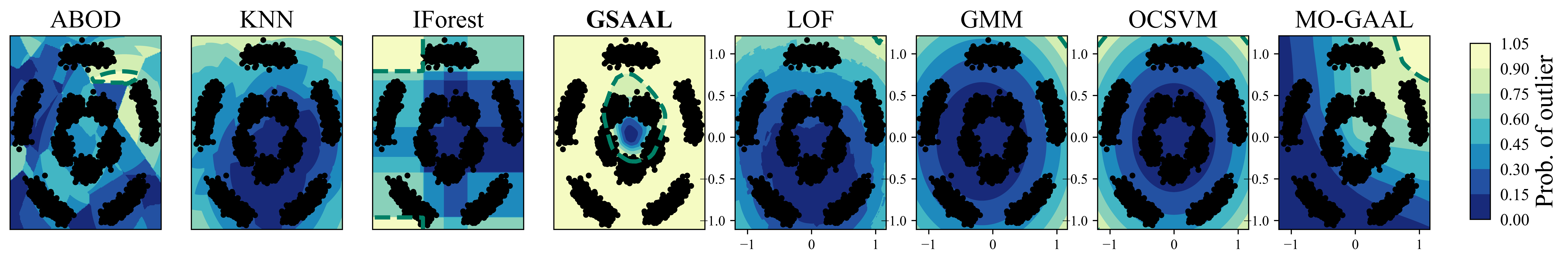}
    \caption{Star}
    \label{fig:center-label}
\end{subfigure}
\begin{subfigure}{1\linewidth}
    \centering
    \includegraphics[width = 1\linewidth]{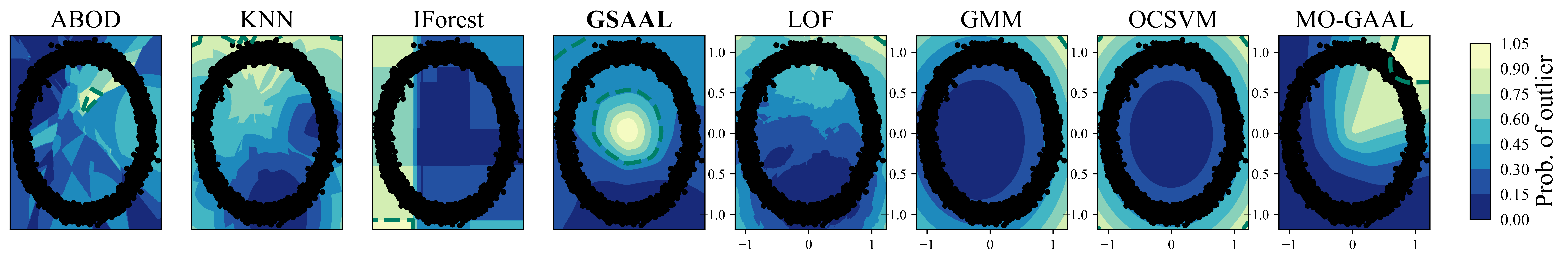}
    \caption{Circle}
    \label{fig:denter-label}
\end{subfigure}
\begin{subfigure}{1\linewidth}
    \centering
    \includegraphics[width = 1\linewidth]{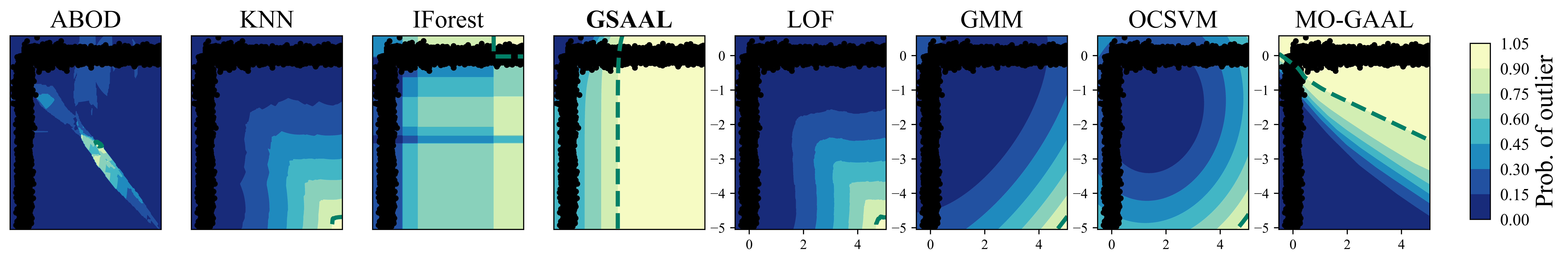}
    \caption{L}
    \label{fig:eenter-label}
\end{subfigure}
\caption{Projected classification boundaries for the datasets in section 4.2 and the extra datasets.}
\label{fig::aMV_appendix}
\end{figure*}

\subsection{One-class Classification (extension)}
As we noted in Section 4, we obtained our benchmark datasets from \cite{Songqiao-2022}, a benchmark study for One-class classification methods in tabular data. Some of the datasets featured in the study, and also in our experiments, were obtained from embedding image or text data using a pre-trained NN (ResNet \cite{He-2015L} and BERT \cite{Devlin-2019}, respectively). We refrain the interested reader into \cite{Songqiao-2022} for additional information. Additionally, we found discrepancies between the versions of the datasets in the study of \cite{Campos-2016} and \cite{Songqiao-2022}. We utilized the version of those datasets featured in \cite{Campos-2016} for our experiments due to popularity. This affected the datasets \emph{Arrhythmia, Annthyroid, Cardiotocography, InternetAds, Ionosphere, SpamBase, Waveform, WPBC} and \emph{Hepatitis}.

Table \ref{tab::occ-AUC} contains all of the AUCs from our experiments. We have included extra methods that, although popular and readily available, were either too similar to other methods in our related work that were more popular (Deep SVDD and INNE) or too far away from our related work (AnoGAN) to be included in the main text. Nevertheless, due to their popularity and inclusion in the libraries used to perform the experiments, we include their results in the appendix for further comparison. Particularly, for INNE we utilized their implementation and default hyperparameters from \texttt{pyod} for the experiments. For DeepSVDD~\cite{ruff-2018} and AnoGAN~\cite{Schlegl-2017}, as there are no official guides on how to fit the models, we tried different training parameter combinations and took the highest obtained AUC. We used their implementation in their official code repository. Whenever a method failed to execute in a particular dataset we denoted it as \textbf{FA}. As it is standard in these studies \cite{Campos-2016,LiuYezheng-2020}, we did not use those datasets subsequent statistical tests. 

\begin{sidewaystable*}
\begin{tabular}{lrrrrrrrrrrrr}
\toprule
Dataset          & \multicolumn{1}{l}{\textbf{GSAAL}} & \multicolumn{1}{l}{LOF} & \multicolumn{1}{l}{IForest} & \multicolumn{1}{l}{ABOD} & \multicolumn{1}{l}{SOD} & \multicolumn{1}{l}{KNN} & \multicolumn{1}{l}{SVDD} & \multicolumn{1}{l}{MO-GAAL} & \multicolumn{1}{l}{GMM} & \multicolumn{1}{l}{DeepSVDD} & \multicolumn{1}{l}{AnoGAN} & \multicolumn{1}{l}{INNE} \\ \midrule
annthyroid       & 0,7681                             & 0,6753                  & 0,7094                      & 0,7008                   & 0,5243                  & 0,6291                  & 0,4611                   & 0,5047                      & 0,6932                  & 0,872                        & 0,4038                     & 0,5081                   \\
Arrhythmia       & 0,7532                             & 0,7277                  & 0,7695                      & 0,7422                   & 0,6514                  & 0,7334                  & 0,7442                   & 0,6901                      & 0,7296                  & 0,7485                       & 0,6133                     & 0,7471                   \\
Cardiotocography & 0,8727                             & 0,8038                  & 0,7772                      & 0,7956                   & 0,3524                  & 0,7733                  & 0,8351                   & 0,7912                      & 0,7413                  & 0,874                        & 0,3248                     & 0,8024                   \\
CIFAR10          & 0,7862                             & 0,7333                  & 0,6853                      & 0,7622                   & 0,6607                  & 0,7493                  & 0,7074                   & 0,6256                      & 0,7462                  & 0,6158                       & 0,3705                     & 0,7306                   \\
FashionMNIST     & 0,8001                             & 0,8995                  & 0,8298                      & 0,9009                   & 0,7136                  & 0,9179                  & 0,8130                   & 0,7930                      & 0,9072                  & 0,6981                       & 0,7137                     & 0,8953                   \\
fault            & 0,6726                             & 0,6436                  & 0,6518                      & 0,8019                   & 0,5670                  & 0,7849                  & 0,5651                   & 0,6821                      & 0,6856                  & 0,4972                       & 0,4074                     & 0,6026                   \\
InternetAds      & 0,7809                             & 0,8565                  & 0,4739                      & 0,8600                   & 0,3663                  & 0,8090                  & 0,7063                   & 0,7603                      & 0,9113                  & 0,8411                       & 0,5165                     & 0,7643                   \\
Ionosphere       & 0,9593                             & 0,9591                  & 0,9377                      & 0,9483                   & 0,8250                  & 0,9825                  & 0,8379                   & 0,9727                      & 0,9644                  & 0,967                        & 0,8406                     & 0,9596                   \\
landsat          & 0,5217                             & 0,7598                  & 0,5927                      & 0,7627                   & 0,4821                  & 0,7726                  & 0,4792                   & 0,4432                      & 0,4998                  & 0,69                         & 0,4835                     & 0,6672                   \\
letter           & 0,6625                             & 0,8888                  & 0,6493                      & \textbf{FA}                       & 0,7182                  & 0,9066                  & 0,9334                   & 0,4828                      & 0,8435                  & 0,676                        & 0,5257                     & 0,7224                   \\
mnist            & 0,7638                             & 0,9484                  & 0,8647                      & 0,9189                   & 0,4858                  & 0,9318                  & \textbf{FA}                       & 0,6151                      & 0,9210                  & 0,7604                       & 0,2502                     & 0,8980                   \\
optdigits        & 0,8935                             & 0,9991                  & 0,8625                      & 0,9846                   & 0,4260                  & 0,9983                  & 0,9999                   & 0,8105                      & 0,8221                  & 0,9086                       & 0,6203                     & 0,9012                   \\
satellite        & 0,8630                             & 0,8456                  & 0,7834                      & \textbf{FA}                       & 0,4745                  & 0,8753                  & 0,8740                   & \textbf{FA}                          & 0,7957                  & 0,7798                       & 0,3099                     & 0,8309                   \\
satimage-2       & 0,9836                             & 0,9966                  & 0,9910                      & 0,9977                   & 0,6745                  & 0,9992                  & 0,9826                   & 0,6317                      & 0,9967                  & 0,9755                       & 0,3968                     & 0,9984                   \\
SpamBase         & 0,8717                             & 0,7132                  & 0,8374                      & 0,7730                   & 0,3774                  & 0,7036                  & 0,6302                   & 0,7377                      & 0,8034                  & 0,7807                       & 0,4826                     & 0,6653                   \\
speech           & 0,6029                             & 0,5075                  & 0,5030                      & 0,8741                   & 0,4364                  & 0,4853                  & 0,4640                   & 0,5138                      & 0,5217                  & 0,6076                       & 0,4821                     & 0,4800                   \\
SVHN            & 0,6859                             & 0,7192                  & 0,5834                      & 0,6989                   & 0,5781                  & 0,6788                  & 0,6150                   & 0,7055                      & 0,6684                  & 0,5894                       & 0,4621                     & 0,6488                   \\
Waveform         & 0,8092                             & 0,7530                  & 0,6902                      & 0,7115                   & 0,5814                  & 0,7623                  & 0,5514                   & 0,6049                      & 0,5791                  & 0,7214                       & 0,7018                     & 0,7562                   \\
WPBC             & 0,6326                             & 0,5695                  & 0,5681                      & 0,6156                   & 0,5333                  & 0,5830                  & 0,5681                   & 0,5972                      & 0,5660                  & 0,4907                       & 0,4121                     & 0,5738                   \\
Hepatitis        & 0,6982                             & 0,5030                  & 0,6568                      & 0,5207                   & 0,2959                  & 0,5680                  & 0,4024                   & \textbf{FA}                          & 0,7574                  & 0,8284                       & 0,3787                     & 0,5325                   \\
MVTec-AD         & 0,9806                             & 0,9679                  & 0,9755                      & 0,9689                   & 0,9662                  & 0,9703                  & 0,9645                   & 0,6412                      & 0,9776                  & 0,7422                       & 0,5179                     & 0,9720                   \\
20newsgroups     & 0,5535                             & 0,7854                  & 0,6675                      & \textbf{FA}                       & 0,7109                  & 0,7260                  & 0,6329                   & 0,5313                      & 0,8103                  & 0,6063                       & 0,4833                     & 0,7074                   \\ \bottomrule
\end{tabular}
\caption{AUC of all the methods tested in section 4.3 and extra methods.}
\label{tab::occ-AUC}
\end{sidewaystable*}

\bibliographystyle{named}
\bibliography{ijcai24}

\end{document}